\documentclass[11pt]{article}

\usepackage{times}
\usepackage{color}
\usepackage{fullpage}
\usepackage{amsthm}
\usepackage[ruled, vlined, linesnumbered]{algorithm2e}
\usepackage{amsmath,amssymb}
\usepackage[colorlinks,linkcolor=blue,citecolor=blue,urlcolor=black]{hyperref}
\usepackage{caption, subcaption}
\DeclareCaptionType{copyrightbox}
\usepackage{framed,url}

\newtheorem{theorem}{Theorem}
\newtheorem{lemma}[theorem]{Lemma}

\newtheorem{claim}[theorem]{Claim}
\newtheorem{corollary}[theorem]{Corollary}

\newtheorem{fact}[theorem]{Fact}

\newcommand{\signid}{{\sc SignId}}

\DeclareMathOperator{\bE}{{\mathop{\mathbb{E}}}}

\newcommand{\eps}{\epsilon}
\newcommand{\E}{\mathcal{E}}

\newcommand{\abs}[1]{\left| #1 \right|}

\newcommand{\A}{\mathcal{A}}

\renewcommand{\O}{\mathcal{O}}

\newcommand{\D}{\mathcal{D}}

\title{Collaborative Learning with Limited Interaction: \\ Tight Bounds for  Distributed Exploration in Multi-Armed Bandits\thanks{Chao Tao is supported in part by NSF IIS-1633215. Qin Zhang is supported in part by NSF IIS-1633215 and CCF-1844234.}}

\author{
Chao Tao \\ 
Computer Science Department\\
Indiana University\\
\texttt{taochao@iu.edu}
\and 
Qin Zhang \\
Computer Science Department\\
Indiana University\\
\texttt{qzhangcs@indiana.edu}
\and 
Yuan Zhou\\
Computer Science Department, Indiana University \\
and\\
Department of ISE, University of Illinois at Urbana-Champaign\\
 \texttt{yuanz@illinois.edu}
}

\begin{document}

\maketitle

\begin{abstract}
Best arm identification (or, pure exploration) in multi-armed bandits is a fundamental problem in machine learning.  In this paper we study the distributed version of this problem where we have multiple agents, and they want to learn the best arm collaboratively.  We want to quantify the power of collaboration under limited interaction (or, communication steps), as interaction is expensive in many settings.  We measure the running time of a distributed algorithm as the {\em speedup} over the best centralized algorithm where there is only one agent.  We give almost tight round-speedup tradeoffs for this problem, along which we develop several new techniques for proving lower bounds on the number of communication steps under time or confidence constraints.  
\end{abstract}

\thispagestyle{empty}
\setcounter{page}{0}
\newpage

\section{Introduction}

One of the biggest challenges in machine learning is to make learning scalable. A natural way to speed up the learning process is to introduce multiple learners/agents, and let them learn the target function collaboratively.  A fundamental question in this direction is to {\em quantify the power of collaboration under limited interaction}, as interaction is expensive in many settings. In this paper we approach this general question via the study of a central problem in online learning -- {\em best arm identification} (or, {\em pure exploration}) in {\em multi-armed bandits}.  We present efficient collaborative learning algorithms and complement them with almost tight lower bounds.

\paragraph{Best Arm Identification.}  In multi-armed bandits (MAB) we have $n$ alternative arms, where the $i$-th arm is associated with an unknown reward distribution $\D_i$ with mean $\theta_i$.  Without loss of generality we assume that each $\D_i$ has support on $[0,1]$; this can always be satisfied with proper rescaling. We also assume that $\theta_i \in [\iota, 1-\iota]$ for any positive constant $\iota > 0$.\footnote{This assumption is due to minor technical reasons, and is also made in many existing bandit lower bounds (e.g. \cite{ABM10}). It does not affect our claims by much, since the most interesting and the hardest instances remain covered by the assumption.} We are interested in the best arm identification problem in MAB, in which we want to identify the arm with the largest mean.  In the standard setting we only have one agent, who tries to identify the best arm by a sequence of arm pulls. Upon each pull of the $i$-th arm the agent observes an {\em i.i.d.}\ sample/reward from $\D_i$.  At any time step, the index of the next pull (or, the final output at the end of the game) is decided by the indices and outcomes of all previous pulls and the randomness of the algorithm (if any).  Our goal is to identify the best arm using the minimum amount of arm pulls, which is equivalent to minimizing the {\em running time} of the algorithm; we can just assume that each arm pull takes a unit time.  


MAB has been studied for more than half a century~\cite{Robbins52,Gittins79}, due to its wide practical applications in clinical trials~\cite{Press2009}, adaptive routings~\cite{AK08}, financial portfolio design~\cite{SWJZ15},  model selection~\cite{MM93}, computer game play~\cite{Silver16}, stories/ads display on website~\cite{ACEMPRRZ08}, just to name a few.  In many of these scenarios we are interested in finding out the best arm (strategy, choice, etc.) as soon as possible and committing to it.  For example, in the Monte Carlo Tree Search used by computer game play engines, we want to find out the best move among a huge number of possible moves. In the task of high-quality website design, we hope to find out the best design among a set of alternatives for display.  In almost all such applications the arm pull is the most expensive component: in the real-time decision making of computer game play, it is time-expensive to perform a single Monte Carlo simulation; in website design tasks, having a user to test each alternative is both time and capital expensive (often a fixed monetary reward is paid for each trial a tester carries out).  

In the literature of best arm identification in MAB, two variants have been considered:
\begin{enumerate}
\item {\em Fixed-time best arm:} Given a time budget $T$, identify the best arm with the smallest error probability.\footnote{In the literature this is often called {\em fixed-budget best arm}. Here we use {\em time} instead of {\em budget} in order to be consistent with the collaborative learning setting, where it is easier to measure the performance of the algorithm by its running time.}

\item {\em Fixed-confidence best arm:}  Given an error probability $\delta$,  identify the best arm with error probability at most $\delta$ using the smallest amount of time.
\end{enumerate}
We will study both variants in this paper.

\paragraph{Collaborative Best Arm Identification.}
In this paper we study best arm identification in the collaborative learning model, where we have $K$ agents who try to learn the best arm together.  The learning proceeds in rounds.  In each round each agent pull a (multi)set of arms {\em without} communication.  For each agent at any time step, based on the indices and outcomes of all previous pulls, all the messages received, and the randomness of the algorithm (if any), the agent, if not in the {\em wait} mode, takes one of the following actions: (1) makes the next pull; (2) requests for a communication step and enters the wait mode;
(3) terminates and outputs the answer.  A communication step starts if {\em all non-terminated} agents are in the wait mode. After a communication step all non-terminated agents exit the wait mode and start a new round. 
During each communication step each agent can broadcast a message to every other agent.  While we do not restrict the size of the message, in practice it will not be too large -- the information of all pull outcomes of an agent can be described by an array of size at most $n$, with each coordinate storing a pair $(c_i, sum_i)$, where $c_i$ is the number of pulls on the $i$-th arm, and $sum_i$ is sum of the rewards of the $c_i$ pulls. 
Once terminated, the agent will not make any further actions.  The algorithm  terminates if all agents terminate.  When the algorithm terminates, each agent should agree on the {\em same} best arm; otherwise we say the algorithm fails.  The {\em number of rounds} of computation, denoted by $R$, is the number of communication steps {\em plus one}.

Our goal in the collaborative learning model is to minimize the number of rounds $R$, and the {\em running time} $T = \sum_{r \in [R]} t_r$, where $t_r$ is the {\em maximum} number of pulls made among the $K$ agents in round $r$.  The motivation for minimizing $R$ is that initiating a communication step always comes with a big time overhead (due to network bandwidth, latency, protocol handshaking), and energy consumption (e.g., think about robots exploring in the deep sea and on Mars). Round-efficiency is one of the major concerns in all parallel/distributed computational models such as the BSP model~\cite{Valiant90} and MapReduce~\cite{DG08}.  The {\em total cost} of the algorithm is a {\em weighted sum} of $R$ and $T$, where the coefficients depend on the concrete applications. We are thus interested in the best {\em round-time tradeoffs} for collaborative best arm identification.

\paragraph{Speedup in Collaborative Learning.}
As the time complexity of the best arm identification in the centralized setting is already well-understood (see, e.g.\ \cite{EMM02,MT04,ABM10,KKS13,JMNB14,KCG16,CL16,CLQ17}), we would like to interpret the running time of a collaborative learning algorithm as the {\em speedup} over that of the best centralized algorithm, which also expresses {\em the power of collaboration}.
Intuitively speaking, if the running time of the best centralized algorithm is $T_{\O}$, and that of a proposed collaborative learning algorithm $\A$ is $T_\A$, then we say the speedup of $\A$ is $\beta_\A = T_\O/T_\A$.  However, due to the parameters in the definition of the best arm identification {\em and} the instance-dependent bounds for the best centralized algorithms, the definition of the speedup of a collaborative learning algorithm needs to be a bit more involved.

Recall that an MAB instance is a set of random variables $\{X_1, \ldots, X_n\}$ each of which has support on $[0,1]$.  Since we are interested in the instance-dependent bounds, we assume that a random permutation is ``built-in'' to the input, that is, the $X_1, \ldots, X_n$ are randomly permuted before being fed to the algorithm. This is a standard assumption in the literature of MAB, since otherwise no conceivable algorithm can achieve instance-optimality -- the foolish algorithm that always outputs the first arm will work perfectly in the instance in which the first arm has the largest mean.

For any fixed-time algorithm $\A$ and an input instance $I$, we let $\delta_{\A} (I, T)$ be the error probability of $\A$ on $I$ given time budget $T$. For any fixed-confidence algorithm $\A$ and an input instance $I$, we let $T_{\A}(I, \delta)$ be the expected time used by $\A$ on $I$ given the confidence parameter $(1-\delta)$. In both definitions, the randomness is taken over both $\A$ and $I$. We also extend the definition $T_{\A}(I, \delta)$ to any fixed-time algorithm $\A$ by letting it be the smallest $T$ such that $\delta_{\A}(I, T) \leq \delta$.

We now define the key notion of \emph{speedup} for a collaborative algorithm. We say that an instance $I$ is $T$-solvable by an algorithm $\O$ (for both fixed-budget and fixed-time and fixed-confidence settings), if $T_{\O}(I, 1/3) \leq T$.  For any $T$, the speedup of a collaborative learning algorithm $\A$ (which can be either fixed-budget or fixed-time) for instances $T$-solvable by a centralized algorithm is defined as follows.
\begin{equation}
\label{def:speedup}
\beta_{\A}(T) = \inf_{\text{centralized }\O} \inf_{\text{instance }I}  \inf_{\delta \in (0, 1/3] : T_{\O}(I, \delta) \leq T}
\frac{T_{\O}(I, \delta)}{T_{\A}(I, \delta)}.
\end{equation}
Here the most inner $\inf$ returns $+\infty$ if the set of candidate $\delta$ is empty. Note that the most natural definition for speedup would be for \emph{all} instances. However, since our upper bound result logarithmically degrades as $T$ grows, we have to introduce the $T$ parameter in the definition, that is, we only consider those instances $I$ for which the centralized algorithm can finish within time $T$ under error $\delta$.

Finally, we let $\beta_{K,R}(T) = \sup_\A \beta_\A(T)$ where the $\sup$ is taken over all $R$-round algorithms $\A$ for the  collaborative learning model with $K$ agents.\footnote{A similar concept of {\em speedup} was introduce in the previous work \cite{HKKLS13}. However, no formal definition was given in \cite{HKKLS13}.}


Clearly there is a tradeoff between $R$ and $\beta_{K,R}$: When $R=1$ (i.e., there is {\em no} communication step), each agent needs to solve the problem by itself, and thus $\beta_{K,1} \le 1$.  When $R$ increases, $\beta_{K,R}$ may increase.  On the other hand we always have $\beta_{K,R} \le K$.  Our goal is to find the best {\em round-speedup} tradeoffs, which is essentially equivalent to the {\em round-time} tradeoffs that we mentioned earlier.  

As one of our goals is to understand the scalability of the learning process, we are particularly interested in one end of the tradeoff curve:
What is the {\em smallest} $R$ such that $\beta_{K,R} = \Omega(K)$?  
In other words, how many rounds are needed to make best arm identification fully scalable in the collaborative learning model? In this paper we will address this question by giving almost tight round-speedup tradeoffs.

\paragraph{Our Contributions.}
\begin{table*}[t]
  \centering
  \begin{tabular}{|c|c|c|c|c|}
	\hline
	problem & number of rounds\footnotemark& $\beta_{K,R}(T)$ & UB/LB & ref. \\
	\hline
	fixed-time & 1 &1 & -- & trivial \\
	& $2$ & $\tilde{\Omega}(\sqrt{K})$ & UB & \cite{HKKLS13} \\
	& $2$ & $\tilde{O}(\sqrt{K})$ & LB & \cite{HKKLS13} \\
	& $R$ & $\tilde{\Omega}(K^{\frac{R-1}{R}})$ & UB & \bf new\\
	& $\Omega\left(\frac{\ln \tilde{K}}{\ln\ln \tilde{K} + \ln\frac{K}{\beta}}\right)$ when $\beta \in [K/\tilde{K}^{0.1}, K]$ & $\beta$ & LB & \bf new\\
	\hline
    fixed-confidence 
    & $R$ & $\tilde{\Omega}\left((\Delta_{\min})^{\frac{2}{R-1}} K\right)$ & UB & \cite{HKKLS13} \\ & $\Omega\left(\min\left\{\frac{\ln \frac{1}{\widetilde{\Delta_{\min}}}}{\ln \left(1 + \frac{K (\ln K)^2}{\beta}\right) + \ln \ln \frac{1}{\widetilde{\Delta_{\min}}}}, \sqrt{\frac{\beta}{(\ln K)^3}}\right\}\right)$ & $\beta$ & LB & \bf new\\
	\hline
  \end{tabular}
  \caption{Our results for collaborative best arm identification in multi-armed bandits. $K$ is the number of agents. $\Delta_{\min}$ is the difference between the mean of the best arm and that of the second best arm in the input. In the lower bound for the fixed-time setting, we set $\tilde{K} = \min\{K, \sqrt{T}\}$; in the lower bound for the fixed-confidence setting, we set $\widetilde{\Delta_{\min}}^{-1} = \min\{\Delta_{\min}^{-1}, T\}$.}
  \label{tab:results}
\end{table*}

\footnotetext{We note again that the number of rounds equals to the number of communication steps  {\em plus one}.}

Our results are shown in Table~\ref{tab:results}. 
For convenience we use the `$\ \tilde{\ }\ $' notation on $O, \Omega, \Theta$ to hide logarithmic factors, which will be made explicit in the actual theorems.
Our contributions include:
\begin{enumerate}
\item {\em Almost tight round-speedup tradeoffs for fixed-time.}  In particular, we show that any algorithm for the fixed-time best arm identification problem in the collaborative learning model with $K$ agents that achieves $(K/\ln^{O(1)} K)$-speedup needs at least $\Omega(\ln K/\ln\ln K)$ rounds (for $T \geq K^{\Omega(1)}$). We complement this lower bound with an algorithm that runs in $\ln K$ rounds and achieves $\tilde{\Omega}(K)$-speedup. 

\item {\em Almost tight round-speedup tradeoffs for fixed-confidence.} In particular, we show that any algorithm for the fixed confidence best arm identification problem in the collaborative learning model with $K$ agents that achieves $(K/\ln^{O(1)} K)$-speedup needs at least $\Omega\left(\ln \frac{1}{\Delta_{\min}}/(\ln\ln K + \ln\ln \frac{1}{\Delta_{\min}})\right)$ rounds (for $T \geq \Delta_{\min}^{-\Omega(1)}$), which almost matches an algorithm in \cite{HKKLS13} that runs in $\ln\frac{1}{\Delta_{\min}}$ rounds and achieves $\tilde{\Omega}(K)$-speedup.  Here $\Delta_{\min}$ is the difference between the mean of the best arm and that of the second best arm in the input.

\item {\em A separation for two problems.} The two results above give a separation on the round complexity of fully scalable algorithms between the fixed-time case and the fixed-confidence case.  In particular, the fixed-time case has smaller round complexity for input instances with $\Delta_{\min} < 1/K$ (and when $T \geq \Delta_{\min}^{-\Omega(1)}$), which indicates that knowing the ``right'' time budget is useful to reduce the number of rounds of the computation. 

\item {\em A generalization of the round-elimination technique.} In the lower bound proof for the fixed-time case, we develop a new technique which can be seen as a generalization of the standard round-elimination technique: we perform the round reduction on {\em classes} of input distributions.  We believe that this new technique will be useful for proving round-speedup tradeoffs for other problems in collaborative learning.


\item {\em A new technique for instance-dependent round complexity.}  In the lower bound proof for the fixed-confidence case, we develop a new technique for proving instance-dependent lower bound for round complexity.  The {\em distribution exchange lemma} we introduce for handling different input distributions at different rounds may be of independent interest.
\end{enumerate}

\paragraph{Related Works.}
There are two main research directions in literature for MAB in the centralized setting, {\em regret minimization} and {\em pure exploration}. In the regret minimization setting (see e.g.\ \cite{ACF02,BC12,LS18}), the player aims at maximizing the total reward gained within the time horizon, which is equivalent to minimizing the {\em regret} which is defined to be the difference between the total reward achieved by the offline optimal strategy (where all information about the input instance is known beforehand) and the total reward by the player. In the pure exploration setting (see, e.g.\ \cite{EMM02,EMM06,ABM10,KKS13,JMNB14,CLQ17}), the goal is to maximize the probability to successfully identify the best arm, while minimizing the number of sequential samples used by the player. Motivated by various applications, other exploration goals were also studied, e.g., to identify the top-$k$ best arms \cite{BWV13,ZCL14,CCZZ17}, and to identify the set of arms with means above a given threshold \cite{LGC16}.


The collaborative learning model for MAB studied in this paper was first proposed by \cite{HKKLS13}, and has proved to be practically useful -- authors of \cite{XZJMH16} and \cite{KZESS17} applied the model to distributed wireless network monitoring and collective sensemaking. 

Agarwal et al.\ \cite{AAAK17} studied the problem of minimum adaptivity needed in pure exploration.  Their model can be viewed as a restricted collaborative learning model, where the agents are {\em not} fully adaptive and have to determine their strategy at the beginning of each round. Some solid bounds on the round complexity are proved in \cite{AAAK17}, including a lower bound using the round elimination technique. As we shall discuss shortly, we develop a generalized round elimination framework and prove a much better round complexity lower bound for a more sophisticated hard instance.

There are other works studying the regret minimization problem under various distributed computing settings. For example, motivated by the applications in cognitive radio network, a line of research (e.g., \cite{LZ10,RSS16,BL18}) studied the regret minimization problem where the radio channels are modeled by the arms and the rewards represent the utilization rates of radio channels which could be deeply discounted if an arm is simultaneously played by multiple agents and a collision occurs. Regret minimization algorithms were also designed for the distributed settings with an underlying communication network for the peer-to-peer environments (e.g., \cite{SBHOJK13,LSL16(1),XTZV15}). In \cite{AK05,CGMM16}, the authors studied distributed regret minimization in the adversarial case. Authors of \cite{PRCS15} studied the regret minimization problem in the batched setting. 

Blum et al.\ \cite{BHPQ17} studied PAC learning of a general function in the collaborative setting, and their results were further strengthened by \cite{CZZ18,NZ18}. However, in the collaborative learning model they studied, each agent can only sample from one particular distribution, and is thus different from the model this paper focuses on.

\section{Techniques Overview}
\label{sec:overview}

In this section we summarize the high level ideas of our algorithms and lower bounds. For convenience, the parameters used in this overview are only for illustration purposes.

\paragraph{Lower bound for fixed-time algorithms.}

A standard technique for proving round lower bounds in communication/sample complexity is the {\em round elimination}~\cite{MNSW98}.  Roughly speaking, we show that if there exists an $r$-round algorithm with error probability $\delta_r$ and sample complexity $f(n_r)$ on an input distribution $\sigma_r$, then there also exists an $(r-1)$-round algorithm with error probability $\delta_{r-1}$ and sample complexity $f(n_{r-1})$ on an input distribution $\sigma_{r-1}$.
Finally, we show that there is {\em no} $0$-round algorithm with error probability $\delta_0 \ll 1$ on a nontrivial input distribution $\sigma_0$.

In \cite{AAAK17} the authors used the round elimination technique to prove an $\Omega(\ln^* n)$ round lower bound for the best arm identification problem under the total pull budget $\tilde{O}({n}/{\Delta^2_{\min}})$.\footnote{$\ln^* n$ is the number of times the logarithm function must be iteratively applied before the result is less than or equal to 1.}  In their hard input there is a single best arm with mean $\frac{1}{2}$, and $(n-1)$ arms with means $(\frac{1}{2} - \Delta_{\min})$.  This ``one-spike'' structure makes it relatively easy to perform the standard round elimination. The basic arguments in \cite{AAAK17} go as follows: Suppose the best arm is chosen from the $n_r = n$ arms uniformly at random. If the agents do not make enough pulls in the first round, then conditioned on the pull outcomes of the first round, the posterior distribution of the index of the best arm can be written as a convex combination of a set of distributions, each of which has support size at least $n_{r-1} \approx \log n$ and is close (in terms of the total variation distance) to the {\em uniform} distribution on its support, and is thus again hard for a $(r-1)$-round algorithm.  

However, since our goal is to prove a much higher {\em logarithmic} round lower bound, we have to restrict the total pull budget within the {\em instance dependent} parameter $\tilde{O}(H) = \tilde{O}\left(\sum_{i=2}^n 1/\Delta_i^2\right)$ ($\Delta_i$ is the difference between the mean of the best arm and that of the $i$-th best arm in the input), and create a hard input distribution with logarithmic levels of arms in terms of their means.\footnote{$H = O(\sum_{i=2}^n {1}/{\Delta_i^2})$ is a standard parameter for describing the pull complexity of algorithms in the multi-armed bandits literature (see, e.g., \cite{BC12}).}  Roughly speaking, we take $\frac{n}{2}$ random arms and assign them with mean $(\frac{1}{2} - \frac{1}{4})$, $\frac{n}{4}$ random arms with mean $(\frac{1}{2} - \frac{1}{8})$, and so on. With such a ``pyramid-like'' structure, it seems difficult to take the same path of arguments as that for the one-spike structure in \cite{AAAK17}. In particular, it is not clear how to decompose the posterior distribution of the means of arms into a convex combination of a set of distributions, each of which is close to the same pyramid-like distribution. We note that such a decomposition is non-trivial even for the one-spike structure.  Now with a pyramid-like structure we have to guarantee that arms of the $(\ell+1)$-th level are chosen randomly from the arms in the union of  the $(\ell+1)$-th level  and  the $\ell$-th level   for each level $\ell$, which looks to be technically challenging.

We take a different approach. We perform the round elimination on {\em classes} of input distributions. More precisely, we show that if there is {\em no} $(r-1)$-round algorithm with error probability $\delta_{r-1}$ and pull complexity $f(n_{r-1})$ on {\em any} distribution in distribution class $\Sigma_{r-1}$, then there is {\em no} $r$-round algorithm with error probability $\delta_r$ and pull complexity $f(n_r)$ on {\em any} distribution in distribution class $\Sigma_r$.  When working with a class of distributions, we do {\em not} need to show that the posterior distribution $\nu'$ of some input distribution $\nu \in \Sigma_r$  is close to a particular distribution, but only that $\nu' \in \Sigma_{r-1}$.  

Although we now have more flexibility on selecting hard input distribution, we still want to find classes of distributions that are easy to work with.  To this end we introduce two more ideas.  First, at the beginning we sample the mean of each arm independently from the {\em same} distribution, in which the pyramid-like structure is encoded.  We found that making the means of arms independent of each other at any time (conditioned on the observations obtained so far) can dramatically simplify the analysis.   Second, we choose to {\em publish} some arms after each round $r$ to make the posterior distribution of the set of unpublished arms stay within the distribution class $\Sigma_{r-1}$. By publishing an arm we mean to exploit the arm and learn its mean exactly.  With the ability of publishing arms we can keep the classes of distributions $\Sigma_r, \Sigma_{r-1}, \ldots$ relatively simple for the round elimination process.

Further different from \cite{AAAK17} in which the set of arms pulled by each agent in each round is pre-determined at the beginning (i.e., the pulls are {\em oblivious} in each round), we allow the agents to act {\em adaptively} in each round.  Allowing adaptivity inside each round adds another layer of technical challenge to our lower bound proof.  Using a {\em coupling}-like argument, we manage to show that when the number of arms $n$ is {\em smaller} than the number of agents $K$, adaptive pulls do not have much advantage against oblivious pulls in each round.  We note that such an argument does not hold when $n \gg K$, and this is why we can only prove a round lower bound of $\Omega(\ln K/\ln\ln K)$ in the adaptive case compared with a round lower bound of $\Omega(\ln n/\ln\ln n)$ in the oblivious case when the speedup $\beta = \tilde{\Omega}(K)$. Surprisingly, this is almost the best that we can achieve -- our next result shows that there is an $\tilde{\Omega}(K)$-speedup adaptive algorithm using $\ln K$ rounds of computation.


\paragraph{Upper bound for fixed-time algorithms.} Our algorithm is conceptually simple, and goes by two phases. The goal of the first phase is to eliminate most of the suboptimal arms and make sure that the number of the remaining arms is at most $K$, which is the number of agents. This is achieved by  assigning each arm to a random agent, and each agent uses $T/2$ time budget to identify the best arm among its assigned arms using the start-of-the-art centralized algorithm. Note that no communication is needed in this phase, and there are still $R$ rounds left for the second phase. We allow each of the $R$ rounds to use $T/(2R)$ time budget. The goal of the $r$-th round in the second phase is to reduce the number of arms to at most $K^{\frac{R-r}{R}}$, so that after the $R$-th round, only the optimal arm survives. To achieve this, we uniformly spend the time budget on each remaining arm. We are able to prove that this simple strategy works, and our analysis crucially relies on the the guarantee that there are at most $K^{\frac{R-r+1}{R}}$ arms at the beginning of the $r$-th round.

We note that when $R=2$, the speedup of our algorithm is $\tilde{\Omega}(\sqrt{K})$, matching that of the $2$-round algorithm presented in \cite{HKKLS13}. Our algorithm also provides the optimal speedup guarantee for $R > 2$, matching our lower bound result mentioned above. 

The algorithm mentioned above only guarantees to identify the best arm with constant error probability. When the input time horizon $T$ is larger, one would expect an algorithm with an error probability that diminishes exponentially in $T$. To this end, we strengthen our basic algorithm to a meta-algorithm that invokes the basic algorithm several times in parallel and returns the plurality vote. One technical difficulty here is that the optimal error probability depends on the input instance and is not known beforehand. One has to guess the right problem complexity and make sure that the basic algorithm does not consistently return the same suboptimal arm when the given time horizon is less than the problem complexity (otherwise the meta algorithm would recognize the suboptimal arm as the best arm with high confidence). 

We manage to resolve this issue via novel algorithmic ideas that may be applied to strengthen fixed-time bandit algorithms in general. In particular, in the first phase of our basic algorithm, we assign a {\em random} time budget (instead of the fixed $T/2$ as described above) to the centralized algorithm invoked by each agent, and this proves to be useful to prevent the algorithm from identifying a suboptimal arm with overwhelmingly high probability.  We note that in \cite{HKKLS13}, the authors got around this problem by allowing the algorithm to have access to {\em both} the time horizon and the confidence parameters, which does not fall into the standard fixed-time category.

\paragraph{Lower bound for fixed-confidence algorithms.}  We first reduce the lower bound for best arm identification algorithms to the task of showing round lower bound for a closely related problem, \signid, which has proved to be a useful proxy in studying the lower bounds for bandit exploration in the centralized setting \cite{Farrell64,JMNB14,CLQ17}. The goal of \signid~is to identify (with fixed confidence) whether the mean reward of the only input arm is greater or less than $1/2$. The difference between $1/2$ and the mean of the arm, denoted by $\Delta$, corresponds to $\Delta_{\min}$ in the best arm identification problem, and our new task becomes to show a round lower bound for the \signid~problem that increases as $\Delta$ approaches $0$. 

While our lower bound proof for fixed-time setting can be viewed as a generalization of the round elimination technique, our lower bound for the \signid~problem in the fixed-confidence setting uses a completely different approach due to the following reasons. First, the online learning algorithm that our lower bound is against  aims at achieving an instance dependent optimal time complexity as it gradually learns the underlying distribution. In other words, the hardness stems from the fact that the algorithm does {\em not} know the underlying distribution beforehand, while traditional round elimination proofs do {\em not} utilize this property. Second, our lower bound proof introduces a sequence of arm distributions and inductively shows that any algorithm needs at least $r$ rounds on the $r$-th input distribution. While traditional round elimination manages to conduct this induction via embedding the $(r-1)$-st input distribution into the $r$-th input distribution, it is not clear how to perform such an embedding in our proof, as our distributions are very different.

Intuitively, in our inductive proof we set the $r$-th input distribution to be the Bernoulli arm with $\Delta = \Delta_r = 1/\zeta^{r}$ and $\zeta > 1$ depends on $K$ (the number of agents) and $\beta$ (the speedup of the algorithm). We hope to show that any algorithm needs $r$ rounds on the $r$-th input distribution. Suppose we have shown the lower bound for the $r$-th input distribution. Since the algorithm has $\beta$-speedup, it performs at most $O(\Delta_{r}^{-2} K/\beta)$ pulls for the $r$-th instance. We will show via a {\em distribution exchange lemma} (which will be explained in details shortly) that this amount of pulls is not sufficient to tell $\Delta = \Delta_{r}$ from $\Delta = \Delta_{r+1}$. Hence the algorithm also uses at most $O(\Delta_{r}^{-2} K/\beta)$ pulls during the first $r$ rounds on the $(r+1)$-st instance, which is not sufficient to decide the sign of the $(r+1)$-st instance. Therefore the algorithm needs at least $(r+1)$ rounds on the $(r+1)$-st instance, completing the induction for the $(r + 1)$-st instance.

To make the intuition rigorous, we need to strengthen our inductive hypothesis as follows. The goal of the $r$-th inductive step is to show that for $\Delta = \Delta_r$, any algorithm needs at least $r$ rounds and makes at most $o(\Delta_r^{-2})$ pulls across the $K$ agents during the first $r$ rounds. While the $0$-th inductive step holds straightforwardly as the induction basis, we go from the $r$-th inductive step to the $(r+1)$-st inductive step via a {\em progress lemma} and the distribution exchange lemma mentioned above.

Given the hypothesis for the $r$-th inductive step, the progress lemma guarantees that the algorithm has to proceed to the $(r+1)$-st round and perform more pulls. Thanks to the strengthened hypothesis, the total number of pulls performed in the first $r$ rounds is $o(\Delta_r^{-2})$. Hence the statistical difference between the pulls drawn from the $r$-th input distribution and its negated distribution (where the outcomes $0$ and $1$ are flipped) is at most $o(1)$ due to Pinsker's inequality, and this is not enough for the algorithm to correctly decide the sign of the arm. 

The distribution exchange lemma guarantees that the algorithm performs no more than $O(\Delta_r^{-2} K /\beta)$ pulls across the agents during the first $(r+1)$ rounds on the $(r+1)$-st input distribution. By setting $\zeta = \omega(K/\beta)$, one can verify that $O(\Delta_r^{-2} K /\beta) = o(\Delta_{r+1}^{-2})$, and the hypothesis for the $(r+1)$-st inductive step is proved. The intuition behind the distribution exchange lemma is as follows. While the algorithm needs $(r+1)$ rounds on the $r$-th input distribution (by the progress lemma), we know that the algorithm cannot use more than $\Omega(\Delta_r^{-2} K/\beta)$ pulls by the $\beta$-speedup constraint. These many pulls are not enough to tell the difference between the $r$-th and the $(r+1)$-st distribution, and hence we can change the underlying distribution and show that the same happens for the $(r+1)$-st input distribution.

However, this intuition is not easy to be formalized. If we simply use the statistical difference between the distributions induced by $\Delta_{r}$ and $\Delta_{r+1}$ to upper bound the probability difference between each agent's behavior for the two input arms, we will face a probability error of $\Theta(\sqrt{1/\beta})$ for each agent. In total, this becomes a probability error of $\Theta(K \sqrt{1/\beta}) \gg 1$ throughout all $K$ agents, which is too much. To overcome this difficulty, we need to prove a more refined probabilistic upper bound on the behavior discrepancy of each agent for different arms. This is achieved via a technical lemma that provides a much better upper bound on the difference between the probabilities that two product distributions assign to the same event, given that the event does not happen very often. This technical lemma may be of independent interest.

\section{Lower Bounds for Fixed-Time Distributed Algorithms}
\label{sec:lb-fixT}

In this section we prove a lower bound for the fixed-time collaborative learning algorithms. We start by considering the non-adaptive case, where in each round each agent fixes the (multi-)set of arms to pull as well as the order of the pulls at the very beginning.  We will then extend the proof to the adaptive case.  

When we write $c= a \pm b$ we mean $c$ is in the range of $[a-b, a+b]$.

\subsection{Lower Bound for Non-Adaptive Algorithms}
\label{sec:lb-fixT-non-adap}

We prove the following theorem in this section.  
\begin{theorem}
\label{thm:lb-1}
For any time budget $T>0$, any $\alpha \in [1, n^{0.2}]$, any $(K/\alpha)$-speedup randomized non-adaptive algorithm for the fixed-time best arm identification problem in the collaborative learning model with $K$ agents and $n \le \sqrt{T}$ arms needs $\Omega(\ln n/(\ln\ln n + \ln\alpha))$ rounds in expectation.
\end{theorem}

\paragraph{Parameters.}  We list a few parameters to be used in the proof. Let $\alpha \in [1, n^{0.2}]$ be the parameter in the statement of Theorem~\ref{thm:lb-1}. Set $B = \alpha (\ln n)^{100}$ (thus $(\ln n)^{100} \le B \le (\ln n)^{100} n^{0.2}$), $\gamma = \alpha (\ln n)^{100}$, $\rho = (\ln n)^{3}$, and $\kappa = (\ln n)^{2}$.

\subsubsection{The Class of Hard Distributions}
\label{sec:distribution}

We first define a class of distributions which is hard for the best arm identification problem.   

Let $L$ be a parameter to be chosen later (in (\ref{eq:L})). Define $\D_j(\eta)$ to be the class of distributions $\pi$ with support 
$$
\{B^{-1}, \ldots, B^{-(j-1)}, B^{-j}, \ldots, B^{-L}\},
$$
such that if $X \sim \pi$, then
\begin{enumerate}
\item
$\Pr\left[(X = B^{-1}) \vee \cdots \vee (X = B^{-(j-1)})\right] \le {n^{-9}},
$ (only defined for $j \ge 2$)
\item For any $\ell = j, \ldots, L$, $\Pr[X = B^{-\ell}] = \lambda_j \cdot B^{-2\ell} \cdot \left(1 \pm \rho^{-\ell} \eta\right)$, where $\lambda_j$ is a normalization factor (to make $\sum_{\ell=1}^L \Pr[X = B^{-\ell}] = 1$). 
\end{enumerate}
Note that when $\eta = 0$, $\D_1(0)$ only contains a single distribution; slightly abusing the notation, define $\D_1 \triangleq \D_1(0)$ to denote that particular distribution. For $j \ge 2$, define $\D_j \triangleq \D_j(\rho^{j-1})$. That is, we set $\eta = \rho^{j-1}$ by default, and consequently $\lambda_j = \left(1 \pm \frac{2}{\rho} \right) B^{2j}$.

We introduce a few threshold parameters: $\zeta_1 =  \left(\frac{1}{2} - B^{-(j+1)}\right)  \gamma B^{2j} - \sqrt{10 \gamma \ln n} B^j$, $\zeta_2 = \frac{\gamma B^{2j}}{2} - B^{j+0.6}$, $\zeta_3 = \frac{\gamma B^{2j}}{2} + B^{j+0.6}$.  It is easy to see that $\zeta_2 < \zeta_1 < \zeta_3$.  

The following lemma gives some basic properties of pulling from an arm with mean $\left(\frac{1}{2} - B^{-\ell}\right)$.  We leave the proof to Appendix~\ref{app:proofs}.
\begin{lemma}
Consider an arm with mean $\left(\frac{1}{2} - X\right)$.   We pull the arm $\gamma B^{2j}$ times. Let $\Theta = (\Theta_1, \Theta_2, \ldots, \Theta_{\gamma B^{2j}})$ be the pull outcomes, and let $\abs{\Theta} = \sum_{i \in [\gamma B^{2j}]} \Theta_i$.   We have the followings.
\begin{enumerate}
\item If $X = B^{-\ell}$ for $\ell > j$, then $\abs{\Theta} \in [\zeta_2, \zeta_3]$ with probability at least $1 - n^{-10}$. \label{item-3}

\item If $X = B^{-\ell}$ for $\ell \le j$, then $\abs{\Theta} < \zeta_1$ with probability at least $1 - n^{-10}$. \label{item-2}

\item If $X = B^{-\ell}$ for $\ell > j$, then $\abs{\Theta} \ge \zeta_1$ with probability at least $1 - n^{-10}$. \label{item-1}
\end{enumerate}
\label{lem:distribution-aux}
\end{lemma}

The next lemma states important properties of distributions in classes $\D_j$. Intuitively, if the mean of an arm is distributed according to some distribution in class $\D_j$, then after pulling it $\gamma B^{2j}$ times,  we can learn by Lemma~\ref{lem:distribution-aux} that at least one of the followings hold: (1) the sequence of pull outcomes is very rare; (2) very likely the mean of the arm is at most $(\frac{1}{2} - B^{-j})$; (3)  very likely the mean of the arm is more than $(\frac{1}{2} - B^{-j})$.  In the first two cases we {\em publish} the arm, that is, we fully exploit the arm and learn its mean exactly.  We will show that if the arm is not published, then the posterior distribution of the mean of the arm (given the outcomes of the $\gamma B^{2j}$ pulls) belongs to class $\D_{j+1}$.  

\begin{lemma}
\label{lem:distribution-class}
Consider an arm with mean $\left(\frac{1}{2} - X\right)$ where $X \sim \mu \in \D_j$ for some $j \in [L-1]$. We pull the arm $\gamma B^{2j}$ times. Let $\Theta = (\Theta_1, \Theta_2, \ldots, \Theta_{\gamma B^{2j}})$ be the pull outcomes, and let $\abs{\Theta} = \sum_{i \in [\gamma B^{2j}]} \Theta_i$.  
If $\abs{\Theta} \not\in [\zeta_1, \zeta_3]$,
then we publish the arm.  Let $\nu$ be the posterior distribution of $X$ after observing $\Theta$. If the arm is not published, then we must have $\nu \in \D_{j+1}$. 
\end{lemma}

\begin{proof}
We analyze the posterior distribution of $X$ after observing $\Theta = \theta$ for any $\theta$ with $\abs{\theta} \in  [\zeta_1, \zeta_3]$.

Let $\chi_{\le j}$ denote the event that $(X = B^{-1}) \vee \cdots \vee (X = B^{-j})$, and let $\chi_{>j}$ denote the event that $(X = B^{-(j+1)}) \vee \cdots \vee (X = B^{-L})$.  Since $X \sim \mu \in \D_j$, we have 
\begin{equation}
\label{eq:c-1}
\Pr[\chi_{>j}] \ge \Pr[X = B^{-(j+1)}] = \left(1 \pm \frac{2}{\rho} \right) B^{2j} \cdot B^{-2(j+1)} \cdot \left(1 \pm \rho^{-(j+1)} \rho^{j-1}\right) \ge 1/(2B^2).
\end{equation}
For the convenience of writing, let $m = \gamma B^{2j}$. Thus $\zeta_1 = m \cdot (\frac{1}{2} - z)$ where $z = B^{-j}\left(B^{-1} + \sqrt{\frac{10 \ln n}{\gamma}}\right)$.
Let $\eps = B^{-j}$, and $\eps' = B^{-(j+1)}$.  

For any $\theta$ with $\abs{\theta} \ge \zeta_1$, we have
\begin{eqnarray}
\Pr[\chi_{\le j} \ |\ \Theta = \theta] 
&=& \frac{\Pr[\Theta = \theta \ |\ \chi_{\le j}] \cdot \Pr[\chi_{\le j}]}{\Pr[\Theta = \theta]} \nonumber \\
&=& \frac{\Pr[\Theta = \theta \ |\ \chi_{\le j}] \cdot \Pr[\chi_{\le j}]}{\Pr[\Theta = \theta\ |\ \chi_{\le j}] \cdot \Pr[\chi_{\le j}] + \Pr[\Theta = \theta\ |\ \chi_{>j}] \cdot \Pr[\chi_{>j}]}\nonumber  \\
&\le&  \frac{\Pr[\Theta = \theta \ |\ X = \eps] \cdot 1}{0 + \Pr[\Theta = \theta\ |\ X = \eps'] \cdot 1/(2B^2)} \quad  (\text{by (\ref{eq:c-1}) and monotonicity}) \nonumber  \\
&=& 2B^2 \cdot \frac{(1/2 - \eps)^{\abs{\theta}} (1/2 + \eps)^{m - \abs{\theta}}}{(1/2 - \eps')^{\abs{\theta}} (1/2 + \eps')^{m - \abs{\theta}}} \nonumber \\
&\le& 2B^2 \cdot \frac{(1/2 - \eps)^{\zeta_1} (1/2 + \eps)^{m - \zeta_1}}{(1/2 - \eps')^{\zeta_1} (1/2 + \eps')^{m - \zeta_1}}
\quad \ (\text{by monotonicity})\nonumber \\
&=& 2B^2  \cdot A^m,   \label{eq:c-2} 
\end{eqnarray}
where 
\begin{equation}
\label{eq:A}
A = \frac{(1-2\eps)^{{1}/{2}-z}(1+2\eps)^{{1}/{2}+z}}{(1-2\eps')^{{1}/{2}-z}(1+2\eps')^{{1}/{2}+z}}.
\end{equation} 
We next analyze $A$. For small enough $\eps > 0$, we have $\eps - \frac{\eps^2}{2} \le \ln(1+\eps) \le \eps - \frac{\eps^2}{2} + \eps^3$, and $-\eps - \frac{\eps^2}{2} - \eps^3 \le \ln(1-\eps) \le -\eps - \frac{\eps^2}{2}$.  Taking the natural logarithm on both sides of (\ref{eq:A}) and using two inequalities for $\ln(1+\eps)$ and $\ln(1-\eps)$ above, we have
\begin{eqnarray}
\ln A &\le& (1/2 - z) \left(-2\eps - 2\eps^2 + 2(\eps') + 2(\eps')^2 + 8(\eps')^3 \right)  + (1/2 + z) \left(2\eps - 2\eps^2 + 8\eps^3 - 2(\eps') + 2(\eps')^2 \right) \nonumber \\
&=& 1/2 \cdot \left(-4\eps^2 + 8\eps^3 + 4 (\eps')^2 + 8(\eps')^3 \right) + z (4\eps + 8\eps^3 - 4(\eps') - 8(\eps')^3) \nonumber \\
&\le& -2 B^{-2j} +  B^{-j}\left(B^{-1} + \sqrt{\frac{10 \ln n}{\gamma}}\right) 4 B^{-j} + O(B^{-2j-1}) \nonumber \\
&\le& - B^{-2j}.  \label{eq:c-3}
\end{eqnarray}
Plugging (\ref{eq:c-3}) back to (\ref{eq:c-2}), we have 
\begin{eqnarray}
\label{eq:c-3-1}
\Pr[\chi_{\le j} \ |\ \Theta = \theta] \le 2B^2 \cdot e^{-B^{-2j} \cdot \gamma B^{2j}} \le n^{-9}.
\end{eqnarray}
where the last inequality holds since $B \le (\ln n)^{100} n^{0.2}$ and $\gamma \ge (\ln n)^{100}$. 
Therefore $\nu$ satisfies the first condition of the distribution class $\D_{j+1}$.

For any $\theta$ with $\abs{\theta} \in [\zeta_1, \zeta_3]$ and $\ell = j+1, \ldots, L$, we have
\begin{eqnarray}
&&\Pr[X = B^{-\ell}\ |\ \Theta = \theta] \nonumber \\
&=& \frac{\Pr[\Theta = \theta\ |\ X = B^{-\ell}] \cdot \Pr[X = B^{-\ell}]}{\Pr[\Theta = \theta]} \nonumber \\
&=& \frac{1}{\Pr[\Theta = \theta]}  \cdot  \left( \Pr\left[\Theta = \bE[\Theta]\ \left|\ X = B^{-\ell} \right.\right] \cdot (1 \pm B^{-\ell})^{B^{j+0.61}}\right)  \cdot \lambda_j B^{-2\ell} \left(1 \pm \rho^{-\ell}\eta \right) \nonumber \\
&=& \frac{1}{\Pr[\Theta = \theta]} \cdot \left( \frac{1}{2\sqrt{2\pi \gamma B^{2j}}} \cdot \frac{1}{\sqrt{1 - 4B^{-2\ell}}}  
 \cdot (1 \pm B^{-\ell})^{B^{j+0.7}} \right)   \cdot \lambda_j B^{-2\ell} \left(1 \pm \rho^{-\ell}\eta \right) \nonumber \\
&=& \left(\frac{1}{\Pr[\Theta = \theta]}  \cdot \frac{1}{2\sqrt{2\pi \gamma B^{2j}}} \cdot \lambda_j \right) \cdot \frac{1}{\sqrt{1 - 4B^{-2\ell}}}  
 \cdot (1 \pm B^{-\ell})^{B^{j+0.7}}  \cdot B^{-2\ell} \left(1 \pm \rho^{-\ell}\eta \right) \nonumber \\
&=& \lambda_j^{'} \cdot (1 \pm 3 B^{-2\ell}) \cdot (1 \pm B^{-\ell+j+0.8}) \cdot B^{-2\ell} \left(1 \pm \rho^{-\ell}\eta \right) \nonumber \\
&=& \lambda_j^{'} \cdot B^{-2\ell} \left(1 \pm \rho^{-\ell} \eta' \right), \label{eq:d-2}
\end{eqnarray}
where 
\begin{itemize}
\item $\lambda'_j$ is a normalization factor.

\item The second equality holds since we have $\abs{\theta} \in [\zeta_1, \zeta_3]$, and thus $\abs{\theta - \bE[\Theta\ |\ X = B^{-\ell}]} \le B^{j+0.61}$. 

\item In the third equality, we have used the Stirling's approximation for factorials (i.e., $n! = \sqrt{2\pi n}\left(\frac{n}{e}\right)^n \left(1+\Theta(\frac{1}{n})\right)$) when calculating $\Pr\left[\Theta = \bE[\Theta]\ \left|\ X = B^{-\ell} \right.\right]$.

\item The fifth inequality holds since $ \frac{1}{\sqrt{1 - 4B^{-2\ell}}}  = 1 \pm 3 B^{-2\ell}$.

\item In the last equality, since $B \ge (\ln n)^{100}$, $\rho =  (\ln n)^3$, $\eta = \rho^{j-1}$ and $\ell \ge j+1$, we can set
$\eta' = \rho^{j}$.  
\end{itemize}
Therefore $\nu$ satisfies the second condition of the distribution class $\D_{j+1}$.

By (\ref{eq:c-3-1}) and (\ref{eq:d-2}), we have $\nu \in \D_{j+1}$.
\end{proof}

\subsubsection{The Hard Input Distribution} 
\label{sec:hard-input-dist}

\paragraph{Input Distribution $\sigma$:} We pick the hard input distribution for the best arm identification problem as follows: the mean of each of the $n$ arms is $\left(\frac{1}{2} - X\right)$, where $X \sim\D_1$.
\medskip

Set $n = B^{2L} / \lambda_1$, where $\lambda_1 = \Theta(B^2)$ is the normalization factor of the distribution $\D_1$. This implies 
\begin{equation}
\label{eq:L}
L = \ln(n\lambda_1)/(2\ln B) = \Theta(\ln n/(\ln \ln n + \ln\alpha)).
\end{equation}

We will use the running time of a good deterministic sequential algorithm as an upper bound for that of any collaborative learning algorithm that has a good speedup.




Let $\E_0$ be the event that there is one and only one best arm with mean $(\frac{1}{2} - B^{-L})$ when $I \sim \sigma$.

\begin{lemma}
\label{lem:W} 
Given budget $W = n \ln^3 n \cdot B^{2}$, the deterministic sequential algorithm in \cite{ABM10} has expected error $o(1)$ on input distribution $\sigma$ conditioned on $\E_0$.
\end{lemma}

\begin{proof}

Given budget $W$, the error of the algorithm in \cite{ABM10} (denoted by $\A_{\mathrm{ABM}}$) on an input instance $I$ is bounded by
\begin{equation}
\label{eq:error-ABM}
\textrm{err}(I) \le n^2 \cdot \exp\left(-\frac{W}{2 \ln n \cdot H(I)}\right),
\end{equation}
where
\begin{equation}
\label{eq:H}
H(I) = \sum_{i = 2}^n \frac{1}{\Delta_i^2} \ ,
\end{equation}
where $\Delta_i$ is the difference between the mean of the best arm and that of the $i$-th best arm in $I$. 
We try to upper bound $H(I)$ when $I \sim \sigma = (\D_1)^n$ conditioned on $\E_0$.


Recall that in the distribution $\D_1$, $\Pr[X = B^{-\ell}] = \lambda_1 B^{-2\ell}$ for $\ell = 1, \ldots, L$ where $\lambda_1 = \Theta(B^2)$ is a normalization factor.  Let $k_\ell$ be the number of arms with mean $(\frac{1}{2} - B^{-\ell})$.
By Chernoff-Hoeffding bound and union bound, we have that with probability $(1 - e^{-B})$, for all $\ell = 1, \ldots, L-1$, 
$$k_\ell = \Theta(\lambda_1 B^{-2\ell} n) = \Theta(B^{2L-2\ell}).$$
Thus for a large enough universal constant $c_H$, with probability $(1 - e^{-B})$,
\begin{equation}
\label{eq:W-1}
H(I) = \sum_{\ell=1}^{L-1} k_\ell \cdot \frac{1}{\left(B^{-\ell} - B^{-L}\right)^2 } \le c_H L B^{2L}.
\end{equation}
Plugging-in (\ref{eq:W-1}) to (\ref{eq:error-ABM}), we get
\begin{equation}
\label{eq:W-2}
\textrm{err}(I) \le n^2 \cdot \exp\left(-\frac{n \ln^3 n \cdot B^2}{2 \ln n \cdot c_H L B^{2L}}\right) = o(1),
\end{equation}
where the equality holds since $n = \Theta(B^{2L}/B^2)$ and $L = O(\ln n / \ln\ln n)$. Therefore, conditioned on $\E_0$ and under time budget $W$, the expected error of $\A_{\mathrm{ABM}}$ on input distribution $\sigma$ is at most $o(1) + e^{-B} = o(1)$.  

\end{proof}


\subsubsection{Proof of Theorem~\ref{thm:lb-1}}

We say a collaborative learning algorithm is $z$-cost if the {\em total} number of pulls made by $K$ agents is $z$.  Since $n \le \sqrt{T}$, we have $W = n \ln^3n \cdot B^2 \le T$. By Lemma~\ref{lem:W} and the definition of speedup (Eq.\ (\ref{def:speedup})),
if there is a $(K/\alpha)$-speedup collaborative learning algorithm, then there must be a $\left(\frac{W}{K/\alpha} \cdot K\right) = (\alpha W)$-cost collaborative learning algorithm that has expected error $o(1)$ on input distribution $\sigma$ conditioned on $\E_0$.  By this observation, Theorem~\ref{thm:lb-1} follows immediately from the following lemma and Yao's Minimax Lemma~\cite{Yao77}.

\begin{lemma}
\label{lem:cost-speedup}
Any deterministic $(\alpha W)$-cost non-adaptive algorithm that solves the best arm identification problem in the collaborative learning model with $K$ agents and $n$ arms with error probability $0.99$ on input distribution $\sigma$ conditioned on $\E_0$ needs $\Omega(\ln n / (\ln\ln n + \ln\alpha))$ rounds.
\end{lemma}

Let  $I_j = \left(\left(1 \pm \frac{1}{L}\right) B^{-2}\right)^{j-1} n$.  In the rest of this section we prove Lemma~\ref{lem:cost-speedup} by induction.  

\paragraph{The Induction Step.}  
The following lemma intuitively states that if there is no good $(r-1)$-round $(\alpha W)$-cost non-adaptive algorithm, then there is no good $r$-round $(\alpha W)$-cost non-adaptive algorithm.
\begin{lemma}
\label{lem:induction}
For any $j \le \frac{L}{2}-1$,
if there is no $(r-1)$-round $(\alpha W)$-cost deterministic non-adaptive algorithm with error probability $\delta + O\left(\frac{1}{\kappa}\right)$ on any input distribution in $(\D_{j+1})^{n_{j+1}}$ for any $n_{j+1} \in I_{j+1}$, then there is no $r$-round $(\alpha W)$-cost deterministic non-adaptive algorithm with error probability $\delta$ on any input distribution in $(\D_{j})^{n_{j}}$ for any $n_j \in I_j$. 
\end{lemma}

\begin{proof}
Consider any $r$-round $(\alpha W)$-cost deterministic non-adaptive algorithm $\A$ that succeeds with probability $\delta'$ on any input distribution in $\mu \in (\D_j)^{n_j}$ for any $n_j \in I_j$. Since we are considering a non-adaptive algorithm, at the beginning of the first round, the total number of pulls by the $K$ agents on each of the $n_j$ arms in the first round are fixed.  Let $(t_1, \ldots, t_{n_j})$ be such a pull configuration, where $t_z$ denotes the number of pulls on the $z$-th arm. 
For an $(\alpha W)$-cost algorithm,
by a simple counting argument,  at least $(1 - \frac{1}{\kappa})$ fraction of $t_z$ satisfies $t_z \le \alpha \kappa \frac{W}{n_j}$.  Let $S$ be the set of  arms $z$ with $t_z > \gamma B^{2j}$.  Since 
\begin{equation*}
 \alpha \kappa \frac{W}{n_j} \le  \alpha \kappa \frac{n \ln^3 n B^2}{\left(\left(1 - \frac{1}{L}\right) B^{-2}\right)^{j-1} n}  \le \gamma B^{2j}, 
\end{equation*}
we have $\abs{S} \le \frac{1}{\kappa} \cdot n_j$.

We augment the first round of Algorithm $\A$ as follows.
\begin{quote}
\noindent{\bf Algorithm Augmentation.}  
\begin{enumerate}
\item We publish all arms in $S$.

\item For the rest of the arms $z \in [n_j] \backslash S$, we keep pulling them until the total number of pulls reaches $\gamma B^{2j}$.  Let $\Theta_z = (\Theta_{z,1}, \ldots, \Theta_{z,\gamma B^{2j}})$ be the  $\gamma B^{2j}$ pull outcomes.  If $\abs{\Theta_z} \not\in [\zeta_1, \zeta_3]$, we publish the arm.

\item If the number of unpublished arms is not in the range of $I_{j+1}$, or there is a published arm with mean $\left(\frac{1}{2} - B^{-L}\right)$, then we return ``error''.
\end{enumerate}
\end{quote}
We note that the first two steps will only help the algorithm, and thus will only lead to a stronger lower bound.  We will show that the extra error introduced by the last step is small, which will be counted in the error probability increase in the induction. 

The following claim bounds the number of arms that are not published after the first round.
\begin{claim}
\label{cla:survive}
For any $j \le \frac{L}{2}-1$, with probability at least $1 - O\left(\frac{1}{\kappa}\right)$, the number of unpublished arms after the first round is in the range $I_{j+1}$.
\end{claim}

\begin{proof}
For each arm $z \in [n_j] \backslash S$, let $\left(\frac{1}{2} - X\right)$ be its mean where $X \sim \pi \in \D_j$.  Let $Y_z$ be the indicator variable of the event that arm $z$ is not published. 
By Lemma~\ref{lem:distribution-aux}, 
\begin{eqnarray*}
\Pr[Y_z = 1] &=& \sum_{\ell > j} \Pr[X = B^{-\ell}] \pm n^{-9} \\
&=& \left(1 \pm \frac{1}{B} \right) \cdot \left(1 \pm \frac{2}{\rho} \right) B^{2j} \cdot B^{-2(j+1)} \left(1 \pm \rho^{-(j+1)} \cdot \rho^{j-1} \right) \pm n^{-9} \\
&=& \left(1 \pm \frac{1}{L^2} \right) \cdot B^{-2},
\end{eqnarray*}
where the second inequality holds since $\Pr[X = B^{-\ell}]$ decreases at a rate of approximately $B^{-2}$ when $\ell$ increments, and  the last inequality holds since $\rho = (\ln n)^3$ and $L < \ln n$.

By Chernoff-Hoeffding bound, and the fact that we publish all arms in $S$, we have  
$$\sum_{z \in [n_j]} Y_z = \left(1 \pm \frac{2}{L^2} \right) B^{-2} (n_j - \abs{S})$$ with probability $1 - e^{-\Omega(n_j (BL)^{-4})} \ge 1 - O\left(\frac{1}{\kappa}\right)$.  Plugging the fact that $\abs{S} \le \frac{1}{\kappa} \cdot n_j$,
we have that with probability $1 - O\left(\frac{1}{\kappa}\right)$ over distribution $\mu$,
$$
\sum_{z \in [n_j]} Y_z = \left(1 \pm \frac{2}{L^2} \right) \left(1 \pm \frac{1}{\kappa} \right) B^{-2} n_j = \left(1 \pm \frac{1}{L} \right) B^{-2} n_j.
$$
Therefore, if $n_j \in I_j$, then with probability $1 - O\left(\frac{1}{\kappa}\right)$, $\sum_{z \in [n_j]} Y_z \in I_{j+1}$.
\end{proof}

The following claim shows that the best arm is not likely to be published in the first round.
\begin{claim}
\label{cla:best-arm}
For any $j \le \frac{L}{2}-1$, the probability that there is a published arm with mean $(\frac{1}{2} - B^{-L})$ is at most $O\left(\frac{1}{\kappa}\right)$.
\end{claim}

\begin{proof}
Since the input distribution to $\A$ belongs to the class $(\D_j)^{n_j}$, 
the probability that $S$ contains an arm with mean $(\frac{1}{2} - B^{-L})$, conditioned on $\abs{S} \le \frac{1}{\kappa} \cdot n_j$, can be upper bounded by
\begin{eqnarray*}
1 - \left(1 -  \lambda_j B^{-2L} \cdot (1+ \rho^{-L+j}) \right)^{\frac{n_j}{\kappa}} 
&\le& 1 - \left(1 -  \lambda_j B^{-2L} \cdot (1+ \rho^{-L+j}) \right)^{\left(\left(1+\frac{1}{L}\right)B^{-2}\right)^{j-1} \cdot \frac{n}{\kappa}} \\
&=& 1 - \left(1 -  \frac{\lambda_j}{B^{2L}} \cdot (1+ \rho^{-L+j}) \right)^{\left(\left(1+\frac{1}{L}\right)B^{-2}\right)^{j-1} \cdot \frac{B^{2L}}{\lambda_1}  \frac{1}{\kappa}} \\
&=& O\left(\frac{1}{\kappa}\right).
\end{eqnarray*}
For each arm $z \in [n] \backslash S$ arms, by Lemma~\ref{lem:distribution-aux} we have that if arm $z$ has mean $(\frac{1}{2} - B^{-L})$, then with probability at least $(1 - n^{-9})$ we have $\abs{\Theta_z} \in [\zeta_1, \zeta_3]$.  The lemma follows by a union bound.
\end{proof}

By Claim~\ref{cla:survive}, Claim~\ref{cla:best-arm} and Lemma~\ref{lem:distribution-class} (which states that if an arm is not published, then its posterior distribution belongs to $\D_{j+1}$), for $j \le \frac{L}{2} - 1$, if there is no $(r-1)$-round $(\alpha W)$-cost algorithm with error probability $\delta'$ on any input distribution in $(\D_{j+1})^{n_{j+1}}$ for any $n_{j+1} \in I_{j+1}$, then there is no $r$-round $(\alpha W)$-cost algorithm with error probability $\left(\delta' - O\left(\frac{1}{\kappa}\right)\right)$  on any input distribution in $(\D_j)^{n_j}$ for any $n_j \in I_j$, which proves Lemma~\ref{lem:induction}.
\end{proof}


\paragraph{The Base Case.}
Recall that in our collaborative learning model, if an algorithm uses $0$ round then it needs to output the answer immediately (without any further arm pull).  We have the following lemma.
\begin{lemma}
\label{lem:base}
Any $0$-round deterministic algorithm must have error probability at least $(1-o(1))$ on any distribution in $(\D_{\frac{L}{2}})^{n_{\frac{L}{2}}}$ (for any $n_{\frac{L}{2}} \in I_{\frac{L}{2}}$) conditioned on $\E_0$. 
\end{lemma}

\begin{proof} First we have
\begin{eqnarray}
n_{\frac{L}{2}} &=& \left(\left(1 \pm \frac{1}{L}\right) B^{-2}\right)^{\frac{L}{2}-1} n =  \left(\left(1 \pm \frac{1}{L}\right) B^{-2}\right)^{\frac{L}{2}-1} \frac{B^{2L}}{B^2} = \Theta(B^L).
\end{eqnarray}
Thus the probability that there exists at least one arm with mean $\left( \frac{1}{2} - B^{-L}\right)$ is 
\begin{eqnarray*}
1 - \left(1 -\left(1 \pm \frac{1}{B}\right) B^{-L} \cdot \left( 1 \pm \rho^{-L} \cdot \rho^{\frac{L}{2}}\right) \right)^{n_{\frac{L}{2}}} = \Theta(1).
\end{eqnarray*}

For each arm $i$ in the $n_{\frac{L}{2}}$ arms, the probability that $i$ and only $i$ has mean $\left(\frac{1}{2} - B^{-L}\right)$ is
\begin{equation*}
\lambda_{\frac{L}{2}} B^{-2L} (1 \pm \rho^{-\frac{L}{2}}) \left(1 - \lambda_{\frac{L}{2}} B^{-2L} (1 \pm \rho^{-\frac{L}{2}})\right)^{n_{\frac{L}{2}}-1} = \Theta\left({1}/{n_\frac{L}{2}}\right).
\end{equation*}
Therefore any $0$-round deterministic algorithm computes the best arm on any distribution in $(\D_{\frac{L}{2}})^{n_{\frac{L}{2}}}$ conditioned on $\E_0$ with probability at most $ O\left({1}/{n_\frac{L}{2}}\right) = o(1)$.
\end{proof}

Lemma~\ref{lem:cost-speedup} follows from Lemma~\ref{lem:induction} and Lemma~\ref{lem:base}.  Note that the extra error accumulated during the induction process is bounded by $L \cdot O\left(\frac{1}{\kappa}\right) = o(1)$ since $L = \Theta(\ln n/(\ln\ln n + \ln\alpha))$.

\subsection{Lower Bound for Adaptive Algorithms}
\label{sec:lb-fixT-adap}

In this section we consider general adaptive algorithms.  We prove the following theorem.

\begin{theorem}
\label{thm:lb-2}
Let $\tilde{K} = \min\{K, \sqrt{T}\}$.
For any $\alpha \in [1, {\tilde{K} }^{0.1}]$, any $(K/\alpha)$-speedup randomized algorithm for the fixed-time best arm identification problem in the collaborative learning model with $K$ agents needs $\Omega(\ln \tilde{K} /(\ln\ln \tilde{K}  + \ln\alpha))$ rounds in expectation.
\end{theorem}


The high level idea for proving Theorem~\ref{thm:lb-2} is the following: We show that adaptivity cannot give much advantage to the algorithm under the input distribution $\sigma$ (defined in Section~\ref{sec:hard-input-dist}) when the number of arms $n$ is smaller than the number of agents $K$.  For this purpose we choose $n$ such that
\begin{equation}
\label{eq:n}
n B^2 = \tilde{K},
\end{equation}
where $\tilde{K} = \min\{K, \sqrt{T}\}$, and $B = \alpha(\ln n)^{100}$ is the parameter defined at the beginning of Section~\ref{sec:lb-fixT-non-adap}. We thus have $n \le \sqrt{T}$, and if $\alpha \le \tilde{K}^{0.1}$ then we have $\alpha \le n^{0.2}$; both conditions are needed if we are going to ``call'' Theorem~\ref{thm:lb-1} (for the non-adaptive case) later in the proof, that is, we will use the proof for the non-adaptive case as a subroutine in the proof for the adaptive case.

We will focus on the case when $\sqrt{T} \ge K$; the proof for the other case is essentially the same.

We make use of the same induction (including notations and the algorithm augmentation) as that for the non-adaptive case in Section~\ref{sec:lb-fixT-non-adap}.  Clearly, the base case (i.e., Lemma~\ref{lem:base}) still holds in the adaptive case since no pull is allowed.

\begin{lemma}
\label{lem:base-2}
Any $0$-round deterministic algorithm must have error probability $1 - o(1)$ on any distribution in $(\D_{\frac{L}{2}})^{n_{\frac{L}{2}}}$ (for any $n_{\frac{L}{2}} \in I_{\frac{L}{2}}$) conditioned on $\E_0$. 
\end{lemma}

Our task is to show the following induction step.  

\begin{lemma}
\label{lem:induction-2}
For any $j \le \frac{L}{2}-1$,
if there is no $(r-1)$-round $(K/\alpha)$-speedup deterministic adaptive algorithm with error probability $\delta + O\left(\frac{1}{\kappa}\right)$ on any input distribution in $(\D_{j+1})^{n_{j+1}}$ for any $n_{j+1} \in I_{j+1}$, then there is no $r$-round $(K/\alpha)$-speedup deterministic adaptive algorithm with error probability $\delta$ on any input distribution in $(\D_{j})^{n_{j}}$ for any $n_j \in I_j$. 
\end{lemma}

We comment that Lemma~\ref{lem:induction-2} does not hold when $n \gg K$ (e.g., $n \ge K^2$), and this is why we can only prove a lower bound of $\Omega(\ln K /(\ln\ln K + \ln\alpha))$ (Theorem~\ref{thm:lb-2}) instead of $\Omega(\ln n /(\ln\ln n + \ln\alpha))$ (Theorem~\ref{thm:lb-1}).  In the rest of this section we prove Lemma~\ref{lem:induction-2}.

\begin{proof}
Let $\E_1$ denote the event that all the $n_j$ arms have means $(\frac{1}{2} - B^{-\ell})$ for $\ell \ge j$.  Since the input is sampled from a distribution in $(\D_j)^{n_j}$, we have 
\begin{equation}
\label{eq:g-1}
\Pr[\E_1] \ge (1 - n^{-9})^{n_j} \ge 1 - n^{-7}.
\end{equation}

Let $(\Theta_1, \ldots, \Theta_t)$ be the outcomes of $t$ pulls when running the adaptive algorithm $\A$ on an input distributed according to $\mu \in (\D_j)^{n_j}$.  We have the following simple fact.

\begin{fact}
\label{fact:uniform}
For any $t \ge 1$, for any possible set of outcomes $(\theta_1, \ldots, \theta_t) \in \{0,1\}^t$, we have
$$\Pr[(\Theta_1, \ldots , \Theta_t) = (\theta_1, \ldots, \theta_t) \ |\ \E_1] = \left(\frac{1}{2} \pm B^{-j}\right)^t.$$
\end{fact}


Let us conduct a thought experiment.  During the run of the adaptive algorithm $\A$, whenever $\A$ pulls an arm, we sample instead an {\em unbiased} coin and let the result be the pull outcome. Let $(\Theta'_1, \ldots ,\Theta'_t)$ be the outcomes of $t$ pulls. It is easy to see that for any $(\theta_1, \ldots, \theta_t) \in \{0,1\}^t$, we have 
\begin{equation}
\label{eq:h-2}
q(\theta_1, \ldots, \theta_t) = \Pr[(\Theta'_1, \ldots ,\Theta'_t) = (\theta_1, \ldots, \theta_t) \ |\ \E_1] = \left(\frac{1}{2}\right)^t.
\end{equation}

In a $(K/\alpha)$-speedup deterministic algorithm $\A$, each agent can make at most $t = \alpha W / K$ pulls.
By Claim~\ref{fact:uniform}, (\ref{eq:h-2}), and the fact that we have set $n = K/B^2$, for any possible pull outcomes $(\theta_1, \ldots, \theta_t) \in \{0,1\}^t$, conditioned on $\E_1$, it holds that
\begin{equation}
\label{eq:h-3}
\frac{p(\theta_1, \ldots, \theta_t)}{q(\theta_1, \ldots, \theta_t)} = \frac{\left(\frac{1}{2} \pm B^{-j}\right)^t}{\left(\frac{1}{2}\right)^t} = (1 \pm 2B^{-j})^{\frac{\alpha W}{K}} = (1 \pm 2B^{-j})^{\alpha \ln^3 n} = \left[\frac{1}{2}, 2\right].
\end{equation}

Let $X_{i,z}$ be the expected number of pulls to arm $z$ by agent $i$ when running $\A$ on input distribution $\mu$. Let $Y_{i,z}$ be the expected number of pulls to arm $z$ by agent $i$ when we we simply feed random $0/1$ outcome to $\A$ at each pull step. By (\ref{eq:h-3}) we have that conditioned on $\E_1$.
\begin{equation}
\label{eq:i-1}
\forall i \in [K], \forall z \in [n_j], \quad \frac{Y_{i,z}}{2} \le X_{i,z} \le 2 Y_{i,z}.
\end{equation}
Since $\sum_{i \in [K]} \sum_{z \in [n_j]} X_{i,z} \le \alpha W$,  conditioned on $\E_1$ we have
\begin{equation}
\label{eq:i-3}
\sum_{i \in [K]} \sum_{z \in [n_j]}  Y_{i,z} \le 2 \alpha W.
\end{equation}

The key observation is that running $\A$ with random $0/1$ pull outcomes is more like running a non-adaptive algorithm. Indeed, we can sample a random bit string of length equal to the number of pulls at the beginning of the algorithm, and then the sequence of indices of arms that will be pulled are fully determined by the random bit string and the decision tree of the deterministic algorithm $\A$. In other words, all $Y_{i,z}$'s can be computed before the run of the algorithm $\A$.

By (\ref{eq:i-3}) and a simple counting argument, conditioned on $\E_1$, we have that for at most ${1}/{\kappa}$ fraction of arms $z \in [n_j]$, it holds that
\begin{equation}
\label{eq:i-4}
\sum_{i \in [K]}  Y_{i,z} \ge \frac{2 \alpha \kappa  W}{n_j}.
\end{equation} 
Denote the set of such $z$'s by $Q$; we thus have $\abs{Q} \le {1}/{\kappa}\cdot n_j$. Note that $Q$ can again be computed before the run of the algorithm $\A$. By (\ref{eq:i-1}) and (\ref{eq:i-4}), we have that conditioned on $\E_1$, for any  $z \in [n_j] \backslash Q$, 
\begin{equation}
\label{eq:i-5}
\sum_{i \in [K]}  X_{i,z} \le \frac{4 \alpha \kappa  W}{n_j} \le \gamma B^{2j}.
\end{equation}

Inequality (\ref{eq:i-5}) tells that for any arms $z \in [n_j] \backslash Q$, 
the total number of pulls on $z$ over the $K$ agents is at most $\gamma B^{2j}$, which is the same as that in the proof for the non-adaptive case in Lemma~\ref{lem:induction} ($Q$ corresponds to $S$ in the proof of Lemma~\ref{lem:induction}). We also have $\Pr[\neg \E_1] \le n^{-7} \le 1/\kappa$ which will contribute to the extra error in the induction. The rest of the proof simply follows from that for Lemma~\ref{lem:induction}.
\end{proof}

\section{Fixed-Time Distributed Algorithms}
\label{secub-fixT}

In this section we present our fixed-time collaborative learning algorithm for the best arm identification problem. The algorithm takes a set $S = [n]$ of $n$ arms, a time horizon $T$, and a round parameter $R$ as input, and is guaranteed to terminate by the $T$-th time step and uses at most $R$ rounds. We assume without loss of generality that $1 \in S$ is the best arm. We state the following theorem as our main algorithmic result.

\begin{theorem}\label{thm:alg-fixed-T}
Let $H = H(I)$ 
be the complexity parameter of the input instance $I$ defined in (\ref{eq:H}).  There exists a collaborative learning algorithm with time budget $T$ and round budget $R$ that returns the best arm with probability at least 
\[
1 - n \cdot \exp\left(- \Omega\left(\frac{TK^{\frac{R-1}{R}}}{H \ln (HK) (\ln (TK^{\frac{R-1}{R}} /H))^2}\right) \right)  .
\]
\end{theorem}


We now show that the algorithm in Theorem~\ref{thm:alg-fixed-T} has $\tilde{\Omega}(K^{\frac{R-1}{R}})$ speedup. 

\begin{theorem}\label{thm:alg-fixed-T-speed-up}
 For any $R \geq 1$, there exists a fixed-time algorithm $\A$ such that $\beta_{\A}(T) = \Omega(K^{\frac{R-1}{R}}\ln (nTK)^{-4})$ for sufficiently large $T$. When $R = \Theta(\ln K)$, the speedup of the algorithm is $\tilde{\Omega}(K)$. 
\end{theorem}
\begin{proof}
It is know \cite{ABM10} that for every instance $I$, it holds that
\[
\inf_{\text{centralized }\O} \delta_{\O}(I, T) \geq \frac{1}{2} \cdot \exp(-O(T/H)) . 
\]
Therefore, for every $\delta \leq 1/3$, we have that
\begin{align}\label{eq:alg-fixed-T-speed-up-1}
\inf_{\text{centralized }\O} T_{\O}(I, \delta) \geq \Omega(H \ln (1/\delta)) . 
\end{align}
On the other hand, let $\A$ be the algorithm in Theorem~\ref{thm:alg-fixed-T}, for $\delta \leq 1/3$, we have that
\begin{align}\label{eq:alg-fixed-T-speed-up-2}
T_{\A}(I, \delta) \leq O\left(H K^{-\frac{R-1}{R}}  \ln (n HK/\delta)^4  \right) .
\end{align}
Combining \eqref{eq:alg-fixed-T-speed-up-1} and \eqref{eq:alg-fixed-T-speed-up-2}, we have
\[
\inf_{\text{centralized }\O} \inf_{\delta \in (0, 1/3] : T_{\O}(I, \delta) \leq T}
\frac{T_{\O}(I, \delta)}{T_{\A}(I, \delta)} \geq \Omega\left(\frac{K^{\frac{R-1}{R}}}{\ln (nTK)^4}\right) ,
\]
which implies that $\beta_{\A}(T) = \Omega(K^{\frac{R-1}{R}}\ln (nTK)^{-4})$ .
\end{proof}

The rest of this section is devoted to the proof of Theorem~\ref{thm:alg-fixed-T}. In Section~\ref{sec:alg-special-case}, we first prove a special case of Theorem~\ref{thm:alg-fixed-T} when $T = \Theta(H K^{-\frac{R-1}{R}} \ln (HK))$, for which the algorithm is guaranteed to output the best arm with constant probability.  Then, in Section~\ref{sec:alg-general}, we prove Theorem~\ref{thm:alg-fixed-T} by performing a technical modification to Algorithm~\ref{alg:fixed-T} and a reduction from general parameter settings to several independent runs of modified Algorithm~\ref{alg:fixed-T} with different parameters.

\subsection{Special Case when $T = \Theta(H K^{-\frac{R-1}{R}} \ln (HK))$}\label{sec:alg-special-case}

Our algorithm for the special case when $T = \Theta(H K^{-\frac{R-1}{R}} \ln (HK))$ is presented in Algorithm~\ref{alg:fixed-T}. We have the following guarantees.

\begin{theorem}\label{thm:alg-fixed-T-basic}
Let $H$ be the instance dependent complexity parameter defined in (\ref{eq:H}). There exists a universal constant $c_{\mathrm{ALG}} > 0$ such that if $T \geq c_{\mathrm{ALG}} H  K^{-\frac{R-1}{R}}\ln (HK)$, then Algorithm~\ref{alg:fixed-T} returns the best arm with probability at least $0.97$.
\end{theorem}

Algorithm~\ref{alg:fixed-T} uses a fixed-confidence centralized procedure $\mathcal{A}_{\mathrm{C}}$ as a building block, with the following guarantees.

\begin{lemma}(See, e.g.\ \cite{EMM06,KKS13,JMNB14,CLQ17})\label{lemma:centralized-alg}
There exists a centralized algorithm $\mathcal{A}_{\mathrm{C}}(I, \delta)$ where $I$ is the input and $\delta$ is the error probability parameter, such that the algorithm returns the best arm and uses at most $O(H(I)( \ln H(I) + \ln \delta^{-1}))$ pulls with probability at least $(1 - \delta)$.
\end{lemma}

\begin{algorithm}[t]
\DontPrintSemicolon
\KwIn{a set of arms $S = [n]$, time horizon $T$ and communication steps $R$ ($R \leq O(\ln K)$)}
initialize $S_0 \leftarrow S$ \\
\For{iteration $r = 1 \text{~to~} R$ } { 
\tcc{Step 1: preparation}
\eIf{$|S_{r - 1}| > K$}{\label{line:alg-fixed-T-3}
randomly assign each arm in $S_{r - 1}$ to one of the $K$ agents, and let $A_{\ell}$ be the set of arms assigned to agent $\ell$ \\
\For{agent $\ell = 1 \text{~to~} K$}{
  $i_\ell^{(r)} \leftarrow \mathcal{A}_{\mathrm{C}}(A_\ell, 0.01)$, if $ \mathcal{A}_{\mathrm{C}}$ does not terminate within $T/2$ pulls, stop the procedure anyways and set $i_\ell^{(r)} \leftarrow \bot$ \\ \label{line:alg-fixed-T-6}
}
}{
assign each arm in $S_{r - 1}$ to $K / |S_{r - 1}|$ agents (so that each agent is assigned with exactly one arm), and let $i_{\ell}^{(r)} $ be the arm assigned to agent $\ell$ \\
}
\tcc{Step 2: learning}
\For{agent $\ell = 1 \text{~to~} K$}{
play arm $i_{\ell}^{(r)}$ for $\frac{1}{2} \cdot T / R$ times and let $\hat{p}_{\ell}^{(r)}$ be the average of the observed rewards  (if $i_{\ell}^{(r)} \neq \bot$)
}
\tcc{Step 3: communication and aggregation}
\For{agent $\ell = 1 \text{~to~} K$}{
broadcast $i_\ell^{(r)}$ and $\hat{p}_\ell^{(r)}$  \\
}
$\tilde{S_r} \leftarrow \{ i_\ell^{(r)} : \ell \in [K] \}$ \\
Let $\hat{q}_i^{(r)} = \frac{1}{|\{ \ell \in [K] : i_\ell = i \}|}\sum_{  \ell \in [K] : i_\ell = i  } \hat{p}_l^{(r)}$  for each $i \in \tilde{S_r}$\\
\tcc{Step 4: elimination}
$S_r \leftarrow \tilde{S_r} \backslash \left\{ i \in \tilde{S_r} : \text{~there~exists~an~arm~} j \mathrm{~with~} \hat{q}_j^{(r)} \geq \hat{q}_i^{(r)} + 2 \cdot \sqrt{\frac{R\ln (200KR)}{\max\{1, K/|S_{r-1}|\} \cdot T}}  \right\}$
}\label{line:alg-fixed-T-15}
\Return the only arm in $S_R$ if $|S_R| = 1$, and $\bot$ otherwise
\caption{Fixed-Time Collaborative Learning Best Arm Identification with Constant Error Probability}
\label{alg:fixed-T}
\end{algorithm}

We describe Algorithm~\ref{alg:fixed-T} briefly in words.
At a high level, the algorithm goes by $R$ iterations. We keep a set of active arms, denoted by $S_{r-1}$, at the beginning of each iteration $r$ with $S_0 = [n]$. During each iteration $r$, the agents collectively learn more information about the active arms in $S_{r-1}$ and eliminate a subset of arms to form $S_r$. This is done in four steps. In the {\em preparation} step, each agent $\ell$ is  assigned with exactly one arm $i_\ell^{(r)}$, which is the one it will learn in the later steps. If there are more agents than active arms, we simply assign each arm to $K/\abs{S_{r-1}}$ agents. Otherwise, we first assign each arm  to a random agent (which can be done by shared randomness without communication), and then each agent uses the centralized procedure $\mathcal{A}_{\mathrm{C}}$ to identify $i_\ell^{(r)}$ as the best arm among the set of assigned arms. We note that the latter case will only happen during iteration $r = 1$ (if it ever happens). Then each agent $\ell$ plays $i_\ell^{(r)}$ in the {\em learning} step and shares his own observation in the {\em communication and aggregation} step. In the {\em elimination} step, we calculate the confidence interval (CI) for each active arm using a carefully designed dependence on $T$, $K$, and $R$, and eliminate the arms whose CI does not overlap with the best arm. We note that this algorithm uses $R$ communication steps, and therefore needs $(R+1)$ rounds. In Section~\ref{sec:alg-special-case-improved}, we describe a trick to shave $1$ communication step and make the algorithm runs in $R$ rounds.

For convenience, we assume without loss of generality that arm $1$ is the best arm in the input set $S$. We first establish the following lemma which concerns about Lines~\ref{line:alg-fixed-T-3}--\ref{line:alg-fixed-T-6} in Algorithm~\ref{alg:fixed-T}.

\begin{lemma}\label{lemma:alg-fixed-T-preparation}
For large enough constant $c_{\mathrm{ALG}} > 0$ and $T \geq c_{\mathrm{ALG}} H K^{-\frac{R-1}{R}} \ln (HK)$, suppose Lines~\ref{line:alg-fixed-T-3}--\ref{line:alg-fixed-T-6} are executed during iteration $r = 1$, then after the preparation step, with probability at least $0.98$, there exists an agent $\ell \in [K]$ such that $i_\ell^{(r)} = 1$.
\end{lemma}
\begin{proof}
Let $\ell^*$ be the agent such that $1 \in A_{\ell^*}$. Since $H(A_{\ell^*}) = \sum_{i \in A_{\ell^*} \setminus \{1\}} \Delta_i^{-2}$. By linearity of expectation, we have that $\bE[H(A_{\ell^*})] = \sum_{i \in S \setminus \{1\}} \Delta_i^{-2}/K = H /K$. By Markov's Inequality, and for large enough $c_{\mathrm{ALG}}$ and $T \geq c_{\mathrm{ALG}} H K^{-\frac{R-1}{R}} \ln H \geq c_{\mathrm{ALG}} H\ln (HK) / K$, we have that with probability at least $0.99$, $T/2$ is greater than or equal to the sample complexity bound in Lemma~\ref{lemma:centralized-alg} for $S = A_{\ell^*}$ and $\delta = 0.01$. Taking a union bound with the event that the run of $\mathcal{A}_{\mathrm{C}}(A_{\ell^*}, 0.01)$ is as described in Lemma~\ref{lemma:centralized-alg}, we have that $\Pr[i_{\ell^*}^{(r)} = 1] \geq 0.98$.

\end{proof}
The following lemma concerns about the learning and elimination steps of Algorithm~\ref{alg:fixed-T}.

\begin{lemma}\label{lemma:alg-fixed-T-elimination}
During each iteration $r$, assuming that $1 \in \tilde{S_r}$, with probability at least $(1 - 0.01/R)$, 
\begin{enumerate}
\item we have that $1 \in S_r$;
\item if we further assume 1) $T \geq c_{\mathrm{ALG}} H K^{-\frac{R-1}{R}} \ln (HK)$ for sufficiently large $c_{\mathrm{ALG}} > 0$ and 2) either $r = 1$ or $|S_{r-1}| \leq K^{\frac{R-r+1}{R}}$, we have that $|S_r| \leq K^{\frac{R-r}{R}}$. 
\end{enumerate}
\end{lemma}
\begin{proof}
Note that for each $i \in \tilde{S_r}$,  we have that $|\{\ell \in [K]: i_{\ell}^{(r)} = i\}| \geq \max\{1, K/|S_{r-1}|\}$. Therefore, $\hat{q}_i^{(r)}$ is the average of at least $\max\{1, K/|S_{r-1}|\} \cdot \frac{1}{2} T/R$ pulls of arm $i$. By Chernoff-Hoeffding bound, we have
\begin{equation}
\label{eq:k-1}
\Pr\left[\left|\hat{q}_i^{(r)} - \theta_i\right| > \sqrt{\frac{R\ln (200KR)}{\max\{1, K/|S_{r-1}|\} \cdot T}}\right] \leq \frac{1}{20 KR} .
\end{equation}

We now condition on the event that $\forall i \in \tilde{S_r}: \left|\hat{q}_i^{(r)} - \theta_i\right| \leq \sqrt{\frac{R\ln (200KR)}{\max\{1, K/|S_{r-1}|\} \cdot T}}$, which holds with probability at least $(1 - 0.01/R)$ by (\ref{eq:k-1}), the fact that $|\tilde{S_r}| \leq K$, and a union bound.  Let $\E_3$ denote this event.

For the first item in the lemma, it is straightforward to verify that $1 \in S_r$ since for any suboptimal arm $i \in \tilde{S_r} \setminus\{1\}$, it holds that
\[
\hat{q}_i^{(r)} - \hat{q}_1^{(r)} \leq \theta_i - \theta_1 + 2\sqrt{\frac{R\ln (200KR)}{\max\{1, K/|S_{r-1}|\} \cdot T}} < 2  \sqrt{\frac{R\ln (200KR)}{\max\{1, K/|S_{r-1}|\} \cdot T}} .
\]

We now show the second item in the lemma. With the additional assumptions (in the second item), we have that $\max\{1, K/|S_{r-1}|\} \geq K^{\frac{r-1}{R}}$. Thus conditioned on $\E_3$, for all arms $i \in \tilde{S_r}$ it holds that
\[
 \left|\hat{q}_i^{(r)} - \theta_i\right| \leq \sqrt{\frac{R\ln (200KR)}{K^{\frac{r-1}{R}} T}} .
\]
For any suboptimal arm $i \in \tilde{S_r}$, the corresponding gap $\Delta_i$ has to be less or equal to $4  \sqrt{\frac{R\ln (200KR)}{K^{\frac{r-1}{R}}T}}$ so that it may stay in $S_r$. This is because otherwise we have
\[
\hat{q}_i^{(r)} + 2 \sqrt{\frac{R\ln (200KR)}{K^{\frac{r-1}{R}} T}}  \leq \theta_i + 3 \sqrt{\frac{R\ln (200KR)}{K^{\frac{r-1}{R}} T}}  \leq \theta_1 -  \sqrt{\frac{R\ln (200KR)}{K^{\frac{r-1}{R}} T}}  \leq \hat{q}_1^{(r)},
\]
and the arm will be eliminated at Line~\ref{line:alg-fixed-T-15}.  Since $H \leq \frac{T K^{\frac{R-1}{R}}}{c_{\mathrm{ALG}} H \ln (HK)}$ and $c_{\mathrm{ALG}}$ is a large enough constant,  the number of suboptimal arms $i$ such that $\Delta_i \leq 4  \sqrt{\frac{R\ln (200KR)}{K^{\frac{r-1}{R}}T}}$  can be upper bounded by 
 \[
 \frac{T K^{\frac{R-1}{R}}}{c_{\mathrm{ALG}} H \ln (HK)} \cdot 16 \cdot \frac{R\ln (200KR)}{K^{\frac{r-1}{R}}T} < K^{\frac{R-r}{R}},
 \] 
and therefore $|S_r| \leq K^{\frac{R-r}{R}}$.
\end{proof}

\begin{proof}[Analysis of Algorithm~\ref{alg:fixed-T}]
By Lemma~\ref{lemma:alg-fixed-T-preparation}, we have $1 \in \tilde{S_1}$ with probability $0.98$, conditioned on which and applying Lemma~\ref{lemma:alg-fixed-T-elimination}, we have both $1 \in S_R$ and $|S_R| \leq 1$ with probability $0.99$ (by a union bound over all $R$ iterations). Therefore, Algorithm~\ref{alg:fixed-T} outputs arm $1$ (the best arm) with probability $0.97$. 
\end{proof}

\subsubsection{Further Improvement on the Round Complexity} \label{sec:alg-special-case-improved}

We have proved that Algorithm~\ref{alg:fixed-T} satisfies the requirement in Theorem~\ref{thm:alg-fixed-T-basic} using $R$ communication steps, and therefore $(R+1)$ rounds. Now we sketch a trick to further reduce the number of communication steps of Algorithm~\ref{alg:fixed-T} by one, and therefore the algorithm only uses $R$ rounds, fully proving Theorem~\ref{thm:alg-fixed-T-basic}.

The main modification is made to the first iteration ($r = 1$) of Algorithm~\ref{alg:fixed-T}. In the preparation step, if $|S_0| > K^{\frac{R-1}{R}}$, then we randomly assign each arm in $S_0$ to $100 K^{\frac{1}{R}}$ agents, and each agent uses the same procedure to identify $i_{\ell}^{(1)}$. Otherwise, the routine of the algorithm remains the same. 

If $|S_0| > K^{\frac{R-1}{R}}$, in the elimination step, we first set $\tilde{S_1}$ to be the set of arms that are identified by at least $K^{\frac{1}{R}}$ agents in the preparation step. Then the elimination rule in Line~\ref{line:alg-fixed-T-15} remains the same.

The rest iterations $r = 2, 3, \dots$ remains the same. However, we only need to proceed to the $(R-1)$-st iteration and therefore the algorithm uses $(R-1)$ communication steps and $R$ rounds.

To analyze the modified algorithm, the main difference is that we can strengthen Lemma~\ref{lemma:alg-fixed-T-preparation} by showing that $1 \in \tilde{S_1}$ with probability at least $0.9$. This is because by Markov's Inequality, for each agent $\ell$ such that $1 \in A_{\ell}$, with probability at least $0.99$, $T/2$ is greater than or equal to the sample complexity of the instance $A_\ell$ (with error probability $\delta = 0.01$)), and therefore $\Pr[i_{\ell}^{(1)} = 1] \geq 0.98$. Therefore, the expected number of agents that identify arm $1$ is at least $0.98 \cdot 100 K^{\frac{1}{R}} \geq 50K^{\frac{1}{R}}$. Applying Markov's Inequality, we show that $\Pr[1 \in \tilde{S_1}] \geq 0.98$. 

We also have that $|\tilde{S_1}| \leq K^{\frac{R-1}{R}}$. Therefore, we iteratively apply a similar argument of Lemma~\ref{lemma:alg-fixed-T-elimination} to the rest of the $(R-1)$ iterations, we have that with probability at least $0.97$, for each $r = 2, 3, \dots, R-1$, it holds that $1 \in S_{r}$ and $|S_r| \leq R^{\frac{R-r-1}{R}}$. Therefore, the algorithm returns arm $1$ after $(R-1)$ iterations with probability at least $0.97$.

\subsection{Algorithm for General Parameter Settings} \label{sec:alg-general}

For conciseness of the presentation, we only extend  Algorithm~\ref{alg:fixed-T} (that uses $(R+1)$ rounds) to general parameter settings. It is easy to verify that the same technique works for the algorithm described in Section~\ref{sec:alg-special-case-improved}, which will fully prove Theorem~\ref{thm:alg-fixed-T}. In the following of this subsection, we prove Theorem~\ref{thm:alg-fixed-T} with an algorithm with round complexity $(R+1)$.

We first make a small modification to Algorithm~\ref{alg:fixed-T} and strengthen its theoretical guarantee. To do this, we need to introduce the following stronger property on the fixed-confidence centralized procedure $\mathcal{A}_{\mathrm{C}}$.

\begin{lemma}\label{lemma:centralized-alg-stronger}
There exists a centralized algorithm $\mathcal{A}_{\mathrm{C}}(S, \delta)$ where the input is a set $S$ of arms, such that there exists a cost function $f_{\mathrm{C}}$ such that
$$f_{\mathrm{C}}(S, \delta) \leq O(H(S) (\ln H(S) + \ln \delta^{-1})),$$ and the function is monotone in inversed gaps $\Delta_2^{-1}, \Delta_3^{-1}, \dots, \Delta_{|S|}^{-1}$ where $\Delta_i$ is the difference between the mean of the best arm and that of the $i$-th best arm, and 
\[
\Pr[\text{algorithm returns the best arm and uses at least $f_{\mathrm{C}}(S, \delta)$ and at most $100 f_{\mathrm{C}}(S, \delta)$ pulls}] \geq 1 - \delta .
\]
\end{lemma}

It can be easily verified that the Successive Elimination algorithm in \cite{EMM06} is a valid candidate algorithm for Lemma~\ref{lemma:centralized-alg-stronger}.

We now describe our technical change to Algorithm~\ref{alg:fixed-T}. 
\begin{quote}
{\bf Algorithm~\ref{alg:fixed-T}$'$}:  In Line~\ref{line:alg-fixed-T-6} of Algorithm~\ref{alg:fixed-T}, instead of choosing $T/2$ as the time threshold, each agent $\ell$ independently chooses $\tau_{\ell} \in \{T/200, T/2\}$ uniformly at random and uses $\tau_{\ell}$ as the time threshold. 
\end{quote}
It is straightforward to see that for a large enough constant $c_{\mathrm{ALG}}$, Theorem~\ref{thm:alg-fixed-T-basic} still holds for the Algorithm~\ref{alg:fixed-T}$'$. We now state the additional guarantee for the  Algorithm~\ref{alg:fixed-T}$'$. 

\begin{lemma}\label{thm:alg-fixed-T-basic-modified}
For any $T$ and any suboptimal arm $i \in S$, the probability that Algorithm~\ref{alg:fixed-T}$'$ returns $i$ is at most $0.86$.
\end{lemma}
\begin{proof}
For any fixed suboptimal arm $i \in S$, let $p$ be the probability that  Algorithm~\ref{alg:fixed-T}$'$  returns $i$. 

If Lines~\ref{line:alg-fixed-T-3}--\ref{line:alg-fixed-T-6} are not executed during iteration $r = 1$ or there exists an agent $\ell$ such that the corresponding $i_{\ell}^{(1)}$ at Line~\ref{line:alg-fixed-T-6} equals to the best arm (arm $1$), by Lemma~\ref{lemma:alg-fixed-T-elimination} we know that  $\Pr[1 \in S_R] \geq 0.99$, and thus the probability that $i$ is returned is at most $0.01$. For now on, we focus on the case that Lines~\ref{line:alg-fixed-T-3}--\ref{line:alg-fixed-T-6} are executed during iteration $r = 1$ and none of  $i_{\ell}^{(1)}$ equals to $1$. 

By Lemma~\ref{lemma:alg-fixed-T-elimination}, we know that $\Pr[\exists \ell: i_{\ell}^{(1)} = i] \geq p - 0.01$. We further have 
$$\Pr[\exists \ell: i_{\ell}^{(1)} = i \text{~and~} \tau_\ell = T/200] \geq p - 0.51$$ 
since $\Pr[\tau_\ell = T/200] = 0.5$. By Lemma~\ref{lemma:centralized-alg-stronger}, we have that 
\begin{equation}
\label{eq:l-1}
\Pr[\exists \ell:  \text{best arm of $A_\ell$ is $i$ and~} f_{\mathrm{C}}(A_\ell, 0.01) \leq T/200] \geq p - 0.52.
\end{equation}

Now consider a new partition of arms $\{A'_\ell\}_{\ell \in [K]}$ which is almost identical to $\{A_\ell\}$ except for that the assignments for arms $1$ and $i$ are exchanged. We note that first, the marginal distribution of $\{A'_\ell\}$ is still the uniform distribution; and second, when $i$ is the best arm of $A_\ell$, we have that $f_{\mathrm{C}}(A_\ell, 0.01) \geq f_{\mathrm{C}}(A'_\ell, 0.01)$ due to the monotonicity of $f_{\mathrm{C}}$ and the gaps of $H(A_\ell)$ are point-wisely less than or equal to that of $H(A'_\ell)$. By (\ref{eq:l-1}), 
\begin{multline*}
\Pr[\exists \ell: \text{best arm of $A_\ell$ is $1$ and~} f_{\mathrm{C}}(A_\ell, 0.01) \leq T/200] \\
~=~ \Pr[\exists \ell: \text{best arm of $A'_\ell$ is $1$ and~} f_{\mathrm{C}}(A'_\ell, 0.01) \leq T/200] 
\geq~ p - 0.52 .
\end{multline*}
By Lemma~\ref{lemma:alg-fixed-T-elimination} and Lemma~\ref{lemma:centralized-alg-stronger}, we have that 
\begin{align*}
\Pr[1 \in S_R]  ~\geq~& \Pr[\exists \ell: i_{\ell}^{(1)} = 1] - 0.01 \\
~\geq~& \Pr[\exists \ell: \text{best arm of $A_\ell$ is $1$ and~} f_{\mathrm{C}}(A_\ell, 0.01) \leq T/200 \text{~and~} \tau_\ell = T/2] - 0.02 \\
~\geq~& \frac{p - 0.52}{2} - 0.02 \\
~=~& \frac{p}{2} - 0.28.
\end{align*}
Since $1 \in S_R$ is a disjoint event from the event that $i$ is returned by the algorithm, we have $p + p/2 - 0.28 \leq 1$, leading to that $p \leq 1.28 / 1.5 < 0.86$ .
\end{proof}

We are now ready to prove the main algorithmic result (Theorem~\ref{thm:alg-fixed-T}).
\begin{proof}[Proof of Theorem~\ref{thm:alg-fixed-T}]
We build a meta algorithm that independently runs the Algorithm~\ref{alg:fixed-T}$'$ for several times with different parameters.

\begin{quote}
{\bf Meta Algorithm}: For each $s = 1, 2, 3, \dots$, we run Algorithm~\ref{alg:fixed-T}$'$ with time horizon $\frac{T}{s^2 10^s} \cdot \frac{6}{\pi^2}$ and communication step parameter $R$ for $10^s$ times, and let the returned values be $i_{s, 1}, i_{s, 2}, \dots, i_{s, 10^s}$. Finally, the algorithm will find the largest $s$ such that the most frequent element in $\{i_{s, \cdot}\}$ has frequency greater than $0.9$ and output the corresponding element, or output $\bot$ if no such $s$ exists. 
\end{quote}
We note that we can still do this in $R$ communication steps and the total run time will be at most 
$$\sum_{s} 10^s \cdot \frac{T}{s^2 10^s} \cdot \frac{6}{\pi^2} \leq T.$$

Let $s^*$ be the largest $s \geq 1$ such that $\frac{T}{s^2 10^s} \cdot \frac{6}{\pi^2}  \geq c_{\mathrm{ALG}} H K^{-\frac{R-1}{R}} \ln (HK)$, where $c_{\mathrm{ALG}}$ is the constant in Theorem~\ref{thm:alg-fixed-T-basic} for  Algorithm~\ref{alg:fixed-T}$'$. If no such $s$ exists, it is easy to verify that the theorem holds trivially. Otherwise, we have that $2^{s^*} =\Omega(TK^{\frac{R-1}{R}} / (H  \ln (HK) (\ln (TK^{\frac{R-1}{R}} /H))^2))$ . 

By Theorem~\ref{thm:alg-fixed-T-basic} and Chernoff-Hoeffding bound, we have that
\[
\Pr[\text{frequency of $1$ in $\{i_{s^*, \cdot}\} > 0.9$}] \geq 1 - \exp\left(-\Omega\left(\frac{TK^{\frac{R-1}{R}}}{H \ln (HK) (\ln (TK^{\frac{R-1}{R}} /H))^2}\right) \right) .
\]
On the other hand, for each $s = s^* + j$ (where $j \geq 1$), by Lemma~\ref{thm:alg-fixed-T-basic-modified}, Chernoff-Hoeffding bound, and a union bound, we have that
\begin{multline*}
\Pr[\exists \text{suboptimal arm $i$}: \text{frequency of $i$ in $\{i_{s, \cdot}\} > 0.9$}] \\ \leq n \cdot  \exp\left(-2^j \cdot \Omega\left(\frac{TK^{\frac{R-1}{R}}}{H \ln (HK) (\ln (TK^{\frac{R-1}{R}} /H))^2}\right) \right) .
\end{multline*}
Finally, we have 
\begin{align*}
&\Pr[\text{Meta Algorithm returns $1$}]\\
 \geq & \Pr[\text{frequency of $1$ in $\{i_{s^*, \cdot}\} > 0.9$}]  - \sum_{j = 1}^{+\infty} \Pr[\exists \text{suboptimal arm $i$}: \text{frequence of $i$ in $\{i_{s^{*} + j, \cdot}\} > 0.9$}]\\
\geq& 1 - \sum_{j=0}^{+\infty}  n \cdot  \exp\left(-2^j \cdot \Omega\left(\frac{TK^{\frac{R-1}{R}}}{H \ln (HK) (\ln (TK^{\frac{R-1}{R}} /H))^2}\right) \right)\\
\geq& 1 - n \cdot \exp\left(- \Omega\left(\frac{TK^{\frac{R-1}{R}}}{H \ln (HK) (\ln (TK^{\frac{R-1}{R}} /H))^2}\right) \right) .
\end{align*}
\end{proof}

\section{Lower Bounds for Fixed-Confidence Distributed Algorithms}
\label{sec:lb-fixC}

In this section, we prove the following lower bound theorem for fixed-confidence collaborative learning algorithms.

\begin{theorem}\label{thm:lb-fixC}
For any large enough $T$, suppose that a randomize algorithm $\A$ for the fixed-confidence best arm identification problem in the collaborative learning model with $K$ agents satisfies that $\beta_{\A}(T) \geq \beta$, then we have that $\A$ uses
\[
\Omega\left(\min\left\{\frac{\min \{\ln (1/\Delta_{\min}), \ln T\}}{\ln (1 + (K (\ln K)^2)/\beta) + \min\{ \ln \ln (1/\Delta_{\min}), \ln \ln T\}}, \sqrt{\beta/(\ln K)^3}\right\}\right)\]
rounds in expectation.
\end{theorem}


To prove the theorem, we work with the following simpler problem.


\paragraph{The \signid~problem.} In the \signid~problem, there is only one Bernoulli arm with mean reward denoted by $(\frac{1}{2} + \Delta)$ (where $\Delta \in [-\frac{1}{2}, \frac{1}{2}] \setminus \{0\}$). The goal for the agent is to make a few pulls on the arm and decide whether $\Delta > 0$ or $\Delta < 0$. Let $I(\Delta)$ denote the input instance. Throughout this section, we use the notations $\Pr_{I(\Delta)}[\cdot]$ and $\bE_{I(\Delta)}[\cdot]$ to denote the probability and expectation when the underlying input instance is $I(\Delta)$. We say a collaborative learning algorithm $\mathcal{A}$ is $\delta$-error and $\beta$-fast for the instance $I(\Delta)$, if we have that
\[
\Pr_{I(\Delta)}\left[\text{$\mathcal{A}$ returns the correct decision within $\Delta^{-2}/\beta$ running time}\right] \geq 1 - \delta .
\]

We first provide the following theorem on the round complexity lower bound for the \signid~problem (which will be formally proved in Section~\ref{sec:proof-thm-lb-fixC-signid}). Then we will show how these statements imply the round complexity lower bound for the best arm identification problem in the fixed confidence setting.

\begin{theorem}\label{thm:lb-fixC-signid}
Let $\Delta^* \in (0, 1/8)$. If $\mathcal{A}$ is a $(1/K^5)$-error and $\beta$-fast algorithm for every \signid~problem instance $I(\Delta)$ where  $|\Delta| \in [\Delta^*, 1/8)$, then there exists $\Delta^{\flat} \geq \Delta^*$ such that
\[
\Pr_{I(\Delta^\flat)}\left[\text{$\mathcal{A}$ uses $\Omega\left(\min\left\{\frac{\ln (1/\Delta^*)}{\ln (1 + K/\beta) + \ln \ln (1/\Delta^*)}, \sqrt{\beta/(\ln K)}\right\}\right)$ rounds}\right] \geq \frac{1}{2} .
\]
\end{theorem}

Since we can easily convert a  $(1/3)$-error and $\beta$-fast algorithm to a $\delta$-error and $\beta/O(\ln \delta^{-1})$-fast algorithm for any $\delta < 0$, we have the following corollary.

\begin{corollary}\label{cor:lb-fixC-signid}
Let $\Delta^* \in (0, 1/8)$. If $\mathcal{A}$ is a $(1/3)$-error and $\beta$-fast algorithm for every \signid~problem instance $I(\Delta)$ where $|\Delta| \in [\Delta^*, 1/8)$, then there exists $\Delta^{\flat} \geq \Delta^*$ such that
\[
\Pr_{I(\Delta^\flat)}\left[\text{$\mathcal{A}$ uses $\Omega\left(\min\left\{\frac{\ln (1/\Delta^*)}{\ln (1 + (K \ln K)/\beta) + \ln \ln (1/\Delta^*)}, \sqrt{\beta/(\ln K)^2}\right\}\right)$ rounds}\right] \geq \frac{1}{2} .
\]
\end{corollary}

We now show how Theorem~\ref{thm:lb-fixC-signid} implies the round complexity lower bound for the best arm identification problem. The proof of our main Theorem~\ref{thm:lb-fixC} will come after the following theorem.

\begin{theorem}\label{thm:lb-fixC-pre}
Let $\Delta^* \in (0, 1/8)$. Given any randomized algorithm $\mathcal{A}_{\mathrm{BAI}}$ for the fixed-confidence best arm identification problem in the collaborative learning model with $K$ agents, if for any $2$-arm instance $J$ where $\Delta_{\min}(J) \in [\Delta^*, 1/8)$, 
\[
\Pr[\text{$\mathcal{A}_{\mathrm{BAI}}$ returns the best arm of $J$ within $\Delta_{\min}^{-2}/\beta$ running time}] \geq \frac{2}{3},
\] 
then there exists a $2$-arm instance $J^*$ where $\Delta_{\min}(J^*) \in [\Delta^*, 1/8)$,  such that
\begin{align}\label{eq:lb-fixC-pre-1}
\Pr\left[\text{$\mathcal{A}_{\mathrm{BAI}}$ uses $\Omega\left(\min\left\{\frac{\ln (1/\Delta^*)}{\ln (1 + (K \ln K)/\beta) + \ln \ln (1/\Delta^*)}, \sqrt{\beta/(\ln K)^2}\right\}\right)$ rounds on  $J^*$}\right] \geq \frac{1}{2}.
\end{align}
\end{theorem}

\begin{proof}
We first show that given such algorithm $\mathcal{A}_{\mathrm{BAI}}$ that uses no more than $R = R(\Delta_{\min})$ rounds of communication in expectation, there exists an algorithm $\mathcal{A}$ for \signid~such that $\mathcal{A}$ is $(1/3)$-error and $\Omega(\beta)$-fast for all instances $I(\Delta)$ where $\Delta \in [\Delta^*, 1/8)$, and $\mathcal{A}$ uses at most $R(\Delta)$ rounds of communication in expectation. 

To construct the algorithm $\mathcal{A}$, we set up a best arm identification instance $J$ where one of the two arms (namely the \emph{reference arm}) is set to be a Bernoulli arm with mean reward $1/2$, and the other arm  (namely the \emph{unknown arm}) is the one in the \signid~instance. $\mathcal{A}$ simulates $\mathcal{A}_{\mathrm{BAI}}$  and plays the arm in the \signid~instance once whenever $\mathcal{A}_{\mathrm{BAI}}$ wishes to play the unknown arm. $\mathcal{A}$ returns `$ < 0$' if and only if $\mathcal{A}_{\mathrm{BAI}}$ returns the reference arm, and $\mathcal{A}$ returns `$> 0$' if and only if $\mathcal{A}_{\mathrm{BAI}}$ returns the unknown arm. 

Suppose  $I(\Delta)$ is the given \signid~instance, we have that $\Delta_{\min}(J) = \Delta$, and therefore $\mathcal{A}$ uses $R(\Delta)$ rounds of communication in expectation. Also one can verify that $\mathcal{A}$ is a $(1/3)$-error and $\beta$-fast algorithm for $I(\Delta)$ whenever $\Delta \in [\Delta^*, 1/8)$. By Corollary~\ref{cor:lb-fixC-signid}, there exists $\Delta^{\flat} \geq \Delta^*$ such that 
\[
\Pr_{I(\Delta^\flat)}\left[\text{$\mathcal{A}$ uses $\Omega\left(\min\left\{\frac{\ln (1/\Delta^*)}{\ln (1 + (K \ln K)/\beta) + \ln \ln (1/\Delta^*)}, \sqrt{\beta/(\ln K)^2}\right\}\right)$ rounds}\right] \geq \frac{1}{2} .
\]
This implies that for the $2$-arm instance $J^*$ where $\Delta_{\min}(J^*) = \Delta^{\flat}$, we have that \eqref{eq:lb-fixC-pre-1} holds.
\end{proof}

\begin{proof}[Proof of Theorem~\ref{thm:lb-fixC}]
Let $J(\Delta)$ be the 2-arm instance where one of the two arms is a Bernoulli arm with mean reward $1/2$ and the other arm is a Bernoulli arm with mean reward $1/2 - \Delta$. By the lil'UCB algorithm in \cite{JMNB14}, we know that there exists a centralized algorithm $\O$ such that $T_{\O}(J(\Delta), 1/3) \leq O(\Delta^{-2} \ln \ln \Delta^{-1})$ for all $\Delta \in (0, 1/4)$. Therefore, there exists a universal constant $c > 0$ such that for any large enough $T$, we have $T_{\O}(J(c T/ \ln T), 1/3) \leq T$ for all $\Delta \in (0, 1/4)$.

For any $\Delta_{\min}$, we set $\Delta^* = \max\{\Delta_{\min}, cT / \ln T\}$. By the definition of $\beta_{\A}(T)$ (in \eqref{def:speedup}) and the assumption that $\beta_{\A}(T) \geq \beta$, we have that for all instance $J(\Delta)$ where $\Delta \in [\Delta^*, 1/4)$, it holds that
\[
\frac{T_{\O}(J(\Delta), 1/3)}{T_{\A}(J(\Delta), 1/3)} \geq \beta,
\] 
which implies that
\[
T_{\A}(J(\Delta), 1/3) \leq \frac{T_{\O}(J(\Delta), 1/3)}{\beta} = O(\Delta^{-2} \ln \ln \Delta^{-1}/\beta) = O(\Delta^{-2} \ln \ln T /\beta) .
\]

We now invoke Theorem~\ref{thm:lb-fixC-pre}, and get that there exists $J^*$ and $\O$ such that $T_{\O}(J^*, 1/3) \leq T$ and
\[
\Pr_{J^*}\left[\text{$\mathcal{A}$ uses $\Omega\left(\min\left\{\frac{\ln (1/\Delta^*)}{\ln (1 + (K \ln K \ln \ln T)/\beta) + \ln \ln (1/\Delta^*)}, \sqrt{\frac{\beta}{(\ln K)^2 \ln \ln T}}\right\}\right)$ rounds}\right] \geq \frac{1}{2}.
\]
Note that $\ln (1/\Delta^*) \leq O(\ln T)$. When $\ln T = \Omega(K)$, the second term in the $\min\{.,.\}$ function becomes smaller. Therefore, in the first term, we can assume that $\ln T = O(K)$ and get the following simplified statement.
\[
\Pr_{J^*}\left[\text{$\mathcal{A}$ uses $\Omega\left(\min\left\{\frac{\ln (1/\Delta^*)}{\ln (1 + (K (\ln K)^2)/\beta) + \ln \ln (1/\Delta^*)}, \sqrt{\frac{\beta}{(\ln K)^3}}\right\}\right)$ rounds}\right] \geq \frac{1}{2}.
\]
\end{proof}

\subsection{Proof of Theorem~\ref{thm:lb-fixC-signid}}\label{sec:proof-thm-lb-fixC-signid}

Suppose $\mathcal{A}$ is a $\delta$-error $\beta$-fast algorithm. We define the following events. For any integer $\alpha \geq 0$, let $\mathcal{E}(\alpha, T)$ to denote the event that $\mathcal{A}$ uses at least $\alpha$ rounds and at most $T$ time steps before the end of the $\alpha$-th round, and let $\mathcal{E}^*(\alpha, T)$ to denote the event that $\mathcal{A}$ uses at least $(\alpha + 1)$ rounds and at most $T$ time steps before the end of the $\alpha$-th round.

We will make use of two lemmas: the {\em progress} lemma and the {\em distribution exchange} lemma. The progress lemma basically says that if the algorithm $\mathcal{A}$ only performs $o(\Delta^2)$ pulls by the end of the $\alpha$-th round, then it must move forward to the $(\alpha+1)$-st round and perform more pulls.

\begin{lemma}[Progress Lemma]\label{lemma:lb-fixC-lemma1}
Recall that  $\mathcal{A}$ is a $\delta$-error $\beta$-fast algorithm, and $\mathcal{E}$ and $\mathcal{E}^*$ are defined at the beginning of this section. For any $\Delta \in [\Delta^*, 1/8)$, any $\alpha \geq 0$, and any $q \geq 1$, so long as 
\[
\Pr_{I(\Delta)}[\mathcal{E}(\alpha, \Delta^{-2}/(Kq)) ] \geq 1/2,
\]
 we have that
\begin{align}\label{eq:lemma-lb-fixC-lemma1}
\Pr_{I(\Delta)}[\mathcal{E}^*(\alpha, \Delta^{-2}/(Kq)) ]  \geq \Pr_{I(\Delta)}[\mathcal{E}(\alpha, \Delta^{-2}/(Kq)) ] - 2\delta - \frac{4}{\sqrt{3q}} ,
\end{align}
where $K$ is the number of agents $\mathcal{A}$ uses in parallel.
\end{lemma}

We defer the proof of Lemma~\ref{lemma:lb-fixC-lemma1} to Section~\ref{sec:lb-fixC-lemma1}. Intuitively, Lemma~\ref{lemma:lb-fixC-lemma1} holds because of the following reason: If $\mathcal{A}$ uses at most $\Delta^{-2}/(Kq)$ time steps, it may perform at most $\Delta^{-2}/q$ pulls throughout all $K$ agents. When $q$ is large, this is not enough information to tell $I(\Delta)$ from $I(-\Delta)$, and therefore $\mathcal{A}$ cannot make a decision on the sign of the arm, and has to proceed to the next round.

The distribution exchange lemma basically says that if the algorithm $\mathcal{A}$ uses $(\alpha + 1)$ rounds for instance $I(\Delta)$, then its $(\alpha + 1)$-st round must conclude before time $\Delta^{-2}/\beta$ for instance $I(\Delta')$ where $\Delta' \leq \Delta$.

\begin{lemma}[Distribution Exchange Lemma]\label{lemma:lb-fixC-lemma2}
Recall that  $\mathcal{A}$ is a $\delta$-error $\beta$-fast algorithm, and $\mathcal{E}$ and $\mathcal{E}^*$ are defined at the beginning of this section. For any $\Delta \in [\Delta^*, 1/8)$, any $\alpha \geq 0$, any $q \geq 100$, and any $\zeta \geq 1$, we have that
\begin{multline}\label{eq:lemma-lb-fixC-lemma2}
\Pr_{I(\Delta /\zeta)}[\mathcal{E}(\alpha + 1, \Delta^{-2}/(Kq) + \Delta^{-2}/\beta) ]\\ \geq  \Pr_{I(\Delta)}[\mathcal{E}^*(\alpha, \Delta^{-2}/(Kq)) ] - \delta - \left(\exp\left(5\sqrt{(3\ln K)/\beta}\right) - 1\right) - 1/K^5 - \frac{8}{\sqrt{3 q}} .
\end{multline}
\end{lemma}

We defer the proof of Lemma~\ref{lemma:lb-fixC-lemma2} to Section~\ref{sec:lb-fixC-lemma2}. At a higher level, we prove Lemma~\ref{lemma:lb-fixC-lemma2} using the following intuition. For instance $I(\Delta)$, since $\mathcal{A}$ is a $\delta$-error $\beta$-fast algorithm, each agent is very likely to use at most $\Delta^{-2}/\beta$ pulls during the $(\alpha + 1)$-st round, and only sees at most $(\Delta^{-2}/(Kq) + \Delta^{-2}/\beta)$ pull outcomes before the next communication (given the event $\mathcal{E}^*(\alpha, \Delta^{-2}/(Kq))$), which is insufficient to tell between $I(\Delta)$ and $I(\Delta/\zeta)$. Therefore, if the instance is $I(\Delta/\zeta)$, each agent is also very likely to use at most $\Delta^{-2}/\beta$ pulls during the $(\alpha + 1)$-st round, and hence the whole algorithm finishes the $(\alpha + 1)$-st round before $(\Delta^{-2}/(Kq) + \Delta^{-2}/\beta)$ time with high probability. 

However, it is not technically easy to formalize this intuition. If we simply use the statistical difference between the two distributions (under $I(\Delta)$ and $I(\Delta/\zeta)$) for the $\Delta^{-2}/\beta$ pulls during the $(\alpha + 1)$-st round to upper bound the probability difference between each agent's behavior for the two instances, we will face a probability error of $\Theta(\sqrt{1/\beta})$ for each agent. In total, this becomes a probability error of $\Theta(K\sqrt{1 / \beta}) \gg 1$ throughout all $K$ agents, which is too much. To overcome this difficulty, in Section~\ref{sec:lb-fixC-useful-lemma}, we establish a technical lemma to derive a much better upper bound on the difference between the probabilities that two product distributions assign to the same event, given that the event does not happen very often. 

We are now ready to prove Theorem~\ref{thm:lb-fixC-signid}.
\begin{proof}[Proof of Theorem~\ref{thm:lb-fixC-signid}]
Combining Lemma~\ref{lemma:lb-fixC-lemma1} and Lemma~\ref{lemma:lb-fixC-lemma2}, when $\Delta \in [\Delta^*, 1/8)$,  $\alpha \geq 0$,  $q \geq 100$,  $\zeta \geq 1$ and $\Pr_{I(\Delta)}[\mathcal{E}(\alpha, \Delta^{-2}/(Kq)) ] \geq 1/2$, we have
\begin{multline}\label{eq:thm-lb-fixC-signid-1}
\Pr_{I(\Delta /\zeta)}[\mathcal{E}(\alpha + 1, \Delta^{-2}/(Kq) + \Delta^{-2}/\beta) ]\\ \geq  \Pr_{I(\Delta)}[\mathcal{E}(\alpha, \Delta^{-2}/(Kq)) ] - 3\delta - \left(\exp\left(5\sqrt{(3\ln K)/\beta}\right) - 1\right) - 1/K^5 - \frac{12}{\sqrt{3 q}}  .
\end{multline}
Set $\zeta = \sqrt{1 + (Kq)/\beta}$,  and \eqref{eq:thm-lb-fixC-signid-1} becomes
\begin{multline}\label{eq:thm-lb-fixC-signid-2}
\Pr_{I(\Delta /\zeta)}[\mathcal{E}(\alpha + 1, (\Delta/\zeta)^{-2}/(Kq)) ]\\ \geq  \Pr_{I(\Delta)}[\mathcal{E}(\alpha, \Delta^{-2}/(Kq)) ] - 3\delta - \left(\exp\left(5\sqrt{(3\ln K)/\beta}\right) - 1\right) - 1/K^5 - \frac{12}{\sqrt{3 q}}  .
\end{multline}

Let $t_0$ be the largest integer such that $ 0.1 \cdot (1 + (K \cdot 1000 t_0^2)/\beta)^{-t_0/2} \geq \Delta^*$, and we have $t_0 = \Omega\left(\frac{\ln (1/\Delta^*)}{\ln (1 + K/\beta) + \ln \ln (1/\Delta^*)}\right)$. Let $t= \min\{t_0, \lfloor c_R \sqrt{\beta/(\ln K)} \rfloor\}$ for some small enough universal constant $c_R > 0$. We also set $q = 1000 t_0^2$. By the definition of event $\mathcal{E}(\cdot, \cdot)$ and the numbering of the steps of the communications, we have that $\mathcal{E}(0, 100/(Kq))$ always holds, and therefore
\begin{align}\label{eq:thm-lb-fixC-signid-3}
1 = \Pr_{I(1/10)}[\mathcal{E}(0, 100/(Kq))] .
\end{align}

 Starting from \eqref{eq:thm-lb-fixC-signid-3}, we iteratively apply \eqref{eq:thm-lb-fixC-signid-2} for $t$ times. Let $\Delta^{\flat} = 0.1 \cdot (1 + (Kq)/\beta)^{-t/2} \geq \Delta^*$, we have that
\begin{align}\label{eq:thm-lb-fixC-signid-4}
\Pr_{I(\Delta^{\flat})}[\mathcal{E}(t, \Delta^{\flat}/(Kq)) ]\geq  1 - \left( 3\delta + \left(\exp\left(5\sqrt{(3\ln K)/\beta}\right) - 1\right) + 1/K^5 + \frac{12}{\sqrt{3000 t_0^2}} \right) t ,
\end{align}
so long as
\begin{align}\label{eq:thm-lb-fixC-signid-5}
 \left( 3\delta + \left(\exp\left(5\sqrt{(3\ln K)/\beta}\right) - 1\right) + 1/K^5 + \frac{12}{\sqrt{3000 t_0^2}} \right) t \leq \frac{1}{2} .
\end{align}
We see that \eqref{eq:thm-lb-fixC-signid-5} holds as long as $\delta \leq 1/K^5$ and $c_R$ is small enough (note that when $\beta < \ln K/c_R^2$ then $t = 0$). Therefore, we conclude that 
\[
\Pr_{I(\Delta^{\flat})}\left[\text{$\mathcal{A}$ uses $\Omega\left(\min\left\{\frac{\ln (1/\Delta^*)}{\ln (1 + K/\beta) + \ln \ln (1/\Delta^*)}, \sqrt{\beta/(\ln K)}\right\}\right)$ rounds}\right] \geq \frac{1}{2} .
\]
\end{proof}

\subsection{Proof of the Progress Lemma (Lemma~\ref{lemma:lb-fixC-lemma1})} \label{sec:lb-fixC-lemma1}

\begin{proof}[Proof of Lemma~\ref{lemma:lb-fixC-lemma1}] Let $F$ denote the event that $\mathcal{A}$ uses exactly $\alpha$ rounds, and uses at most $\Delta^{-2}/(Kq)$ time steps. It is clear that 
$$\Pr_{I(\Delta)}[\mathcal{E}^*(\alpha, \Delta^{-2}/(Kq)) ]  \geq \Pr_{I(\Delta)}[\mathcal{E}(\alpha, \Delta^{-2}/(Kq)) ] - \Pr_{I(\Delta)} [F].$$ Therefore it suffices to show that 
\begin{align}\label{eq:lb-fixC-lemma1-1}
\Pr_{I(\Delta)} [F] \leq  2\delta +\frac{4}{\sqrt{3q}} .
\end{align}

Note that 
\begin{align}\label{eq:lb-fixC-lemma1-2}
\Pr_{I(\Delta)} [F] = \Pr_{I(\Delta)} [F \wedge \text{$\mathcal{A}$ returns `$> 1/2$'}] + \Pr_{I(\Delta)} [F \wedge \text{$\mathcal{A}$ returns `$< 1/2$'}] .
\end{align}
We first focus on the first term of the Right-Hand Side (RHS) of \eqref{eq:lb-fixC-lemma1-2}. Let $\mathcal{D}_{\Delta}$ denote the product distribution $\mathcal{B}(1/2 + \Delta)^{\otimes \Delta^{-2}/q}$, and let  $\mathcal{D}_{-\Delta}$ denote $\mathcal{B}(1/2 - \Delta)^{\otimes \Delta^{-2}/q}$, where $\mathcal{B}(\theta)$ is the Bernoulli distribution with the expectation $\theta$. By Pinsker's inequality (Lemma~\ref{lem:pinsker}) and simple KL-divergence calculation we have that when $\Delta \in (0, 1/8)$, it holds that
\[
\|\mathcal{D}_{\Delta} - \mathcal{D}_{-\Delta}\|_{\mathrm{TV}} \leq \sqrt{\frac{1}{2} \mathrm{KL}(\mathcal{D}_{\Delta}\| \mathcal{D}_{-\Delta})} \leq \frac{4}{\sqrt{3q}}.
\]

On the other hand, since when event $F$ happens, $\mathcal{A}$ uses at most $\Delta^{-2}/(Kq) \cdot K = \Delta^{-2}/q$ pulls (over all agents), we have 
\begin{align}\label{eq:lb-fixC-lemma1-3}
 \Pr_{I(\Delta)} [F \wedge \text{$\mathcal{A}$ returns `$> 1/2$'}]~ \leq~&  \Pr_{I(-\Delta)} [F \wedge \text{$\mathcal{A}$ returns `$> 1/2$'}] + \|\mathcal{D}_{\Delta} - \mathcal{D}_{-\Delta}\|_{\mathrm{TV}}   \nonumber \\
~\leq ~&  \Pr_{I(-\Delta)} [ \text{$\mathcal{A}$ returns `$> 1/2$'}] +  \frac{4}{\sqrt{3q}} \leq \delta +  \frac{4}{\sqrt{3q}}. 
\end{align}

For the second term of the RHS of \eqref{eq:lb-fixC-lemma1-2}, we have
\begin{align}\label{eq:lb-fixC-lemma1-4}
 \Pr_{I(\Delta)} [F \wedge \text{$\mathcal{A}$ returns `$< 1/2$'}]  \leq  \Pr_{I(\Delta)} [ \text{$\mathcal{A}$ returns `$< 1/2$'}]  \leq \delta .
\end{align}

Combining \eqref{eq:lb-fixC-lemma1-2}, \eqref{eq:lb-fixC-lemma1-3}, and \eqref{eq:lb-fixC-lemma1-4}, we prove \eqref{eq:lb-fixC-lemma1-1}.
\end{proof}

\subsection{Probability Discrepancy under Product Distributions for Infrequent Events} \label{sec:lb-fixC-useful-lemma}

In this section, we prove the following lemma to upper bound the difference between the probabilities that two product distributions assign to the same event. Given that the event does not happen very often, our upper bound is significantly better than the total variation distance between the two product distributions.

\begin{lemma}\label{lemma:lb-fixC-useful-lemma}
Suppose $0 \leq \Delta' \leq \Delta \leq 1/8$. For any positive integer $m = \Delta^{-2}/\xi$ where $\xi \geq 100$, let $\mathcal{D}$ denote the product distribution $\mathcal{B}(1/2 + \Delta)^{\otimes m}$ and let $\mathcal{D}'$ denote the product distribution $\mathcal{B}(1/2 + \Delta')^{\otimes m}$, where $\mathcal{B}(\mu)$ is the Bernoulli distribution with the expectation $\mu$. Let $\mathcal{X}$ be any probability distribution with sample space $X$. For any event $A \subseteq \{0, 1\}^m \times X$ such that $\Pr_{\mathcal{D} \otimes \mathcal{X}}[A] \leq \gamma$, we have that
\[
\Pr_{\mathcal{D}' \otimes \mathcal{X}}[A] \leq \gamma\cdot\exp\left(5 \sqrt{(3\ln Q)/\xi}\right) + 1/Q^6 ,
\]
holds for all $Q \geq \xi$.
\end{lemma}
\begin{proof}
Let $L = \{\ell \in \{0, 1\}^m : |\ell| \geq m/2 - z/\Delta\}$ where $|\ell|$ denotes the number of $1$'s in the vector $\ell$ and $z \geq 0$ is a parameter to be decided later. We have that
\begin{align}\label{eq:lb-fixC-useful-lemma-1}
\Pr_{(\ell, x) \sim \mathcal{D}' \otimes \mathcal{X}}[(\ell, x) \in A] \leq  \Pr_{(\ell, x) \sim \mathcal{D}' \otimes \mathcal{X}}[(\ell, x) \in A \wedge \ell \in L] + \Pr_{\ell \sim \mathcal{D}'} [\ell \not\in L] .
\end{align}
We first focus on the first term of the RHS of \eqref{eq:lb-fixC-useful-lemma-1}. Note that
\begin{align}\label{eq:lb-fixC-useful-lemma-2}
 \Pr_{(\ell, x) \sim \mathcal{D}' \otimes \mathcal{X}}[(\ell, x) \in A \wedge \ell \in L] = \sum_{\ell \in L} \Pr_{x \sim \mathcal{X}} [(\ell, x) \in A\  |\ \ell \in L ] \cdot (1/2+\Delta')^{|\ell|} (1/2-\Delta')^{m-|\ell|}
\end{align}
When $\ell \in L$, by monotonicity, we have
\begin{align}\label{eq:lb-fixC-useful-lemma-3}
\frac{(1/2+\Delta')^{|\ell|} (1/2-\Delta')^{m-|\ell|}}{(1/2+\Delta)^{|\ell|} (1/2-\Delta)^{m-|\ell|}} ~\leq ~&\frac{(1/2+\Delta')^{m/2-z/\Delta} (1/2-\Delta')^{m/2+z/\Delta}}{(1/2+\Delta)^{m/2-z/\Delta} (1/2-\Delta)^{m/2+z/\Delta}}  \nonumber\\
=~ &\left(\frac{1/4-(\Delta')^2}{1/4-\Delta^2}\right)^{m/2} \left(\frac{(1/2-\Delta')(1/2+\Delta)}{(1/2+\Delta')(1/2-\Delta)}\right)^{z/\Delta} \nonumber\\
\leq~& \left(\frac{1}{1-4\Delta^2}\right)^{m/2} \left(\frac{1+2\Delta}{1-2\Delta}\right)^{z/\Delta} .
\end{align}
Since $(1 - \epsilon)^{-1/\epsilon} \leq e^{1.2}$ for all $\epsilon \in (0, 1/4)$ and $(1 + \epsilon)^{1/\epsilon} \leq e$ for all $\epsilon \in (0, 1)$, for $\Delta \in (0, 1/8)$, we have
\begin{align}\label{eq:lb-fixC-useful-lemma-4}
\left(\frac{1}{1-4\Delta^2}\right)^{m/2} \left(\frac{1+2\Delta}{1-2\Delta}\right)^{z/\Delta} ~ \leq ~&\exp\left(1.2 \cdot 4\Delta^2 \cdot m/2 +1.2 \cdot 2\Delta \cdot z/\Delta + 2\Delta \cdot z/\Delta\right) \nonumber\\
= ~& \exp(2.4/\xi + 4.4z) .
\end{align}
Combining \eqref{eq:lb-fixC-useful-lemma-2}, \eqref{eq:lb-fixC-useful-lemma-3}, \eqref{eq:lb-fixC-useful-lemma-4}, we have
\begin{align}\label{eq:lb-fixC-useful-lemma-5}
\Pr_{(\ell, x) \sim \mathcal{D}' \otimes \mathcal{X}}[(\ell, x) \in A \wedge \ell \in L]  ~\leq~& \exp(2.4/\xi + 4.4z) \cdot \Pr_{(\ell, x) \sim \mathcal{D} \otimes \mathcal{X}}[(\ell, x) \in A \wedge \ell \in L] \nonumber\\
\leq~& \exp(2.4/\xi + 4.4z) \cdot \Pr_{(\ell, x) \sim \mathcal{D} \otimes \mathcal{X}}[(\ell, x) \in A ] \nonumber\\
 \leq ~&\gamma \cdot  \exp(2.4/\xi + 4.4z) .
\end{align}

For the second term of the RHS of \eqref{eq:lb-fixC-useful-lemma-1}, by Chernoff-Hoeffding bound, we have
\begin{align}\label{eq:lb-fixC-useful-lemma-6}
\Pr_{\ell \sim \mathcal{D}'} [\ell \not\in L]  \leq \exp\left(-2 m (z/(\Delta m) )^2\right) = \exp\left(-2 z^2 \xi\right) .
\end{align}
\end{proof}

Combining \eqref{eq:lb-fixC-useful-lemma-1}, \eqref{eq:lb-fixC-useful-lemma-5}, and \eqref{eq:lb-fixC-useful-lemma-6}, we have 
\[
\Pr_{(\ell, x) \sim \mathcal{D}' \otimes \mathcal{X}}[(\ell, x) \in A] \leq \gamma \cdot  \exp(2.4/\xi + 4.4z) +  \exp\left(-2 z^2 \xi\right) .
\]
Setting $z = \sqrt{(3\ln Q)/\xi}$ and for $\xi \geq 100$ and $Q \geq \xi$, we have
\[
\Pr_{(\ell, x) \sim \mathcal{D}' \otimes \mathcal{X}}[(\ell, x) \in A] \leq \gamma \cdot \exp\left(2.4/\xi + 4.4 \sqrt{(3\ln Q)/\xi}\right) + 1/Q^6  \leq \gamma\cdot\exp\left(5 \sqrt{(3\ln Q)/\xi}\right) + 1/Q^6.
\]

\subsection{Proof of the Distribution Exchange Lemma (Lemma~\ref{lemma:lb-fixC-lemma2})} \label{sec:lb-fixC-lemma2}

We first introduce a simple mathematical lemma, whose proof can be found in Appendix~\ref{app:proofs}.
\begin{lemma} \label{lem:use-ineq}
For any $\gamma_1, \dots, \gamma_K \in [0, 1]$ and $x \geq 0$ , it holds that
\[
\prod_{i = 1}^K \max\{1 - \gamma_i - \gamma_i x, 0\} \geq \prod_{i = 1}^K (1 - \gamma_i) - x.
\]
\end{lemma}

\begin{proof}[Proof of Lemma~\ref{lemma:lb-fixC-lemma2}]
We will only prove \eqref{eq:lemma-lb-fixC-lemma2} for  $\mathcal{A}$ as a deterministic algorithm, i.e.\ when there is no randomness in $\mathcal{A}$ except for the observed rewards drawn from the arm. Once this is established, we can easily deduce that the same inequality holds for randomized $\mathcal{A}$ by taking expectation on both sides of \eqref{eq:lemma-lb-fixC-lemma2} over the (possibly shared) random bits used by each agent of the collaborative learning algorithm $\mathcal{A}$.

Let $\ell \in \{0, 1\}^{\Delta^{-2}/q}$ be the rewards from the first ${\Delta^{-2}/q}$ plays of the arm. Once conditioned on $\ell$, $\mathcal{E}^*(\alpha, \Delta^{-2}/(Kq))$ becomes a deterministic event, since $\mathcal{A}$ is deterministic and the event only depends on the first ${\Delta^{-2}/q}$ rewards. In light of this, we let $\mathfrak{S}$ denote the set of $\ell$ conditioned on which $\mathcal{E}^*(\alpha, \Delta^{-2}/(Kq))$ holds. We have
\begin{align}\label{eq:lb-fixC-lemma2-1}
\sum_{s \in \mathfrak{S}} \Pr_{I(\Delta)}[\ell = s] =  \Pr_{I(\Delta)}[\mathcal{E}^*(\alpha, \Delta^{-2}/(Kq)) ] .
\end{align}
For each agent $i \in [K]$, let $G_i$ be the event that the agent uses more than $\Delta^{-2}/\beta$ pulls during the $(\alpha + 1)$-st round. Since $\mathcal{A}$ is deterministic, conditioned on $\ell \in \mathfrak{S}$, $G_i$ only depends on the random rewards observed by the $i$-th agent during the $(\alpha + 1)$-st  round, and is independent from $G_j$ for any $j \neq i$. Since $\mathcal{A}$ is a $\delta$-error $\beta$-fast algorithm, we have
\begin{align}
 \delta ~\geq ~&\Pr_{I(\Delta)} [\text{$\mathcal{A}$ uses $> \Delta^{-2}/\beta$ time}]  \nonumber \\
 \geq ~& \sum_{s \in \mathfrak{S}}  \Pr_{I(\Delta)}[\ell = s] \cdot \Pr_{I(\Delta)}[G_1 \vee G_2 \vee \dots \vee G_K \ |\ \ell = s] \nonumber \\
 = ~& \sum_{s \in \mathfrak{S}}  \Pr_{I(\Delta)}[\ell = s] \cdot \left(1-\prod_{i=1}^{K} \left(1 -\Pr_{I(\Delta)}[G_i \ |\ \ell = s]\right)\right) \nonumber \\
 = ~ &\Pr_{I(\Delta)}[\mathcal{E}^*(\alpha, \Delta^{-2}/(Kq)) ] - \sum_{s \in \mathfrak{S}}  \Pr_{I(\Delta)}[\ell = s] \cdot\prod_{i=1}^{K} \left(1 -\Pr_{I(\Delta)}[G_i \ |\ \ell = s]\right), \nonumber
\end{align}
where the last equality is because of \eqref{eq:lb-fixC-lemma2-1}.  We thus have
\begin{equation}
\label{eq:lb-fixC-lemma2-2a}
\sum_{s \in \mathfrak{S}}  \Pr_{I(\Delta)}[\ell = s] \cdot\prod_{i=1}^{K} \left(1 -\Pr_{I(\Delta)}[G_i \ |\ \ell = s]\right) \ge  \Pr_{I(\Delta)}[\mathcal{E}^*(\alpha, \Delta^{-2}/(Kq)) ] - \delta 
\end{equation}

We also have
\begin{align}\label{eq:lb-fixC-lemma2-2}
\Pr_{I(\Delta /\zeta)}[\mathcal{E}(\alpha + 1, \Delta^{-2}/(Kq) + \Delta^{-2}/\beta) ]  
\geq ~&\sum_{s \in \mathfrak{S}}  \Pr_{I(\Delta/\zeta)}[\ell = s] \cdot  \Pr_{I(\Delta/\zeta)} [ \neg G_1 \wedge \neg G_2 \wedge \dots \wedge \neg G_K \ |\ \ell = s] \nonumber \\
= ~& \sum_{s \in \mathfrak{S}}  \Pr_{I(\Delta/\zeta)}[\ell = s] \cdot  \prod_{i=1}^{K}  \left(1 -\Pr_{I(\Delta/\zeta)} [ G_i \ |\ \ell = s] \right) .
\end{align}

We next to fuse (\ref{eq:lb-fixC-lemma2-2a}) and (\ref{eq:lb-fixC-lemma2-2}).
Invoking Lemma~\ref{lemma:lb-fixC-useful-lemma} with $Q = K$ and $\xi = \beta$, we have
\begin{align}
 ~& \sum_{s \in \mathfrak{S}}  \Pr_{I(\Delta/\zeta)}[\ell = s] \cdot   \prod_{i=1}^{K}\left(1 - \Pr_{I(\Delta/\zeta)} [ G_i \ |\ \ell = s] \right) \nonumber \\
 \geq~ & \sum_{s \in \mathfrak{S}}  \Pr_{I(\Delta/\zeta)}[\ell = s] \cdot  \prod_{i=1}^{K} \max\left\{1 -\Pr_{I(\Delta)} [ G_i \ |\ \ell = s] \cdot \exp\left(5\sqrt{(3\ln K)/\beta}\right) - 1/K^6, 0 \right\}  \nonumber \\
 \geq~ & \sum_{s \in \mathfrak{S}}  \Pr_{I(\Delta/\zeta)}[\ell = s] \cdot  \left(\prod_{i=1}^{K} \max\left\{1 -\Pr_{I(\Delta)} [ G_i \ |\ \ell = s] \cdot \exp\left(5\sqrt{(3\ln K)/\beta}\right), 0 \right\}  - 1/K^5 \right) \nonumber \\
 \geq~ & \sum_{s \in \mathfrak{S}}  \Pr_{I(\Delta/\zeta)}[\ell = s] \cdot  \left(\prod_{i=1}^{K} \left(1 -\Pr_{I(\Delta)} [ G_i \ |\ \ell = s] \right) - \left(\exp\left(5\sqrt{(3\ln K)/\beta}\right) - 1\right) - 1/K^5 \right) \nonumber\\
\geq~ & \sum_{s \in \mathfrak{S}}  \Pr_{I(\Delta/\zeta)}[\ell = s] \cdot  \prod_{i=1}^{K} \left(1 -\Pr_{I(\Delta)} [ G_i \ |\ \ell = s] \right) - \left(\exp\left(5\sqrt{(3\ln K)/\beta}\right) - 1\right) - 1/K^5,  \label{eq:lb-fixC-lemma2-4}
\end{align}
where the second to the last inequality is due to Lemma~\ref{lem:use-ineq}.
Finally, we have
\begin{multline}\label{eq:lb-fixC-lemma2-5}
 \sum_{s \in \mathfrak{S}}  \Pr_{I(\Delta/\zeta)}[\ell = s] \cdot  \prod_{i=1}^{K} \left(1 -\Pr_{I(\Delta)} [ G_i \ |\ \ell = s] \right) \\
 \geq  \sum_{s \in \mathfrak{S}}  \Pr_{I(\Delta)}[\ell = s] \cdot  \prod_{i=1}^{K} \left(1 -\Pr_{I(\Delta)} [ G_i \ |\ \ell = s] \right) - \sum_{s \in \mathfrak{S}} \left|  \Pr_{I(\Delta/\zeta)}[\ell = s] -   \Pr_{I(\Delta)}[\ell = s]  \right| ,
\end{multline}
where by Pinsker's inequality (Lemma~\ref{lem:pinsker}) and simple KL-divergence calculation for $\Delta \in (0, 1/8)$, we have 
\begin{align}\label{eq:lb-fixC-lemma2-6}
 \sum_{s \in \mathfrak{S}} \left|  \Pr_{I(\Delta/\zeta)}[\ell = s] -   \Pr_{I(\Delta)}[\ell = s]  \right| \leq \frac{8}{\sqrt{3 q}} .
 \end{align}
 
 Combining \eqref{eq:lb-fixC-lemma2-2},  \eqref{eq:lb-fixC-lemma2-4}, \eqref{eq:lb-fixC-lemma2-5}, and \eqref{eq:lb-fixC-lemma2-6}, we have
\begin{align}\label{eq:lb-fixC-lemma2-7}
& \Pr_{I(\Delta /\zeta)}[\mathcal{E}(\alpha + 1, \Delta^{-2}/(Kq) + \Delta^{-2}/\beta) ]\nonumber\\
  \geq ~&\sum_{s \in \mathfrak{S}}  \Pr_{I(\Delta)}[\ell = s] \cdot  \prod_{i=1}^{K} \left(1 -\Pr_{I(\Delta)} [ G_i \ |\ \ell = s] \right)  - \left(\exp\left(5\sqrt{(3\ln K)/\beta}\right) - 1\right) - 1/K^5 - \frac{8}{\sqrt{3 q}} \nonumber \\
\geq ~& \Pr_{I(\Delta)}[\mathcal{E}^*(\alpha, \Delta^{-2}/(Kq)) ] - \delta - \left(\exp\left(5\sqrt{(3\ln K)/\beta}\right) - 1\right) - 1/K^5 - \frac{8}{\sqrt{3 q}},
 \end{align}
where the last inequality is due to \eqref{eq:lb-fixC-lemma2-2a}.
\end{proof}

\bibliographystyle{plain}
\bibliography{paper}

\appendix

\section{Probability Tools}

The following lemma states Chernoff-Hoeffding bound. 
\begin{lemma}\label{thm:hoeffding}
Let $X_1, X_2, \dots, X_n$ be independent random variables bounded by $[0, 1]$. Let ${X} = \sum_{i = 1}^{n} X_i$. For additive error, for every $t \geq 0$, it holds that
\begin{align*}
    \Pr\left[\abs{{X} - \bE[{X}]} \geq t\right] \leq  2\exp\left(-\frac{2t^2}{n} \right).
\end{align*}
For multiplicative error, for every $\delta \in [0, 1]$, it holds that
\begin{align*}
    \Pr\left[\abs{{X}- \bE[X]} \geq \delta \bE[X] \right] \leq  2\exp\left({-\frac{\delta^2 \bE[X]}{3}}\right).
\end{align*}
\end{lemma}

The following lemma states Pinsker's inequality \cite{Pinsker64}.

\begin{lemma}\label{lem:pinsker}
If $P$ and $Q$ are two discrete probability distributions on a measurable space $(X, \Sigma)$, then for any measurable event $A \in \Sigma$, it holds that
\[
\left| P(A) - Q(A) \right| \leq \sqrt{\frac{1}{2} \mathrm{KL}(P \| Q)} 
\]
where 
\[ 
\mathrm{KL}(P \| Q) = \sum_{x \in X} P(x) \ln \left( \frac{P(x)}{Q(x)}\right)
\]
is the Kullback--Leibler divergence.
\end{lemma}

\section{Missing Proofs}
\label{app:proofs}

\subsection{Proof of Lemma~\ref{lem:distribution-class}}

\begin{proof}
Let $S_\ell = \abs{\Theta}_{| X = B^{-\ell}}$.  We have
$\bE[S_\ell] = \gamma B^{2j} \cdot \left(\frac{1}{2} - B^{-\ell}\right)$.

For the first item, we have for any $\ell > j$,
\begin{equation*}
\bE[S_\ell] = \gamma B^{2j} \cdot \left(\frac{1}{2} - B^{-\ell}\right) = \frac{\gamma B^{2j}}{2} - \gamma B^{2j - \ell} = \frac{\gamma B^{2j}}{2} \pm \gamma B^{j-1}.
\end{equation*}
Since $B = \gamma \ge (\ln n)^{100}$, by Chernoff-Hoeffding bound we have that for any $\ell > j$, with probability at least $1 - n^{-10}$,
\begin{equation*}
\label{eq:d-1}
S_\ell = \frac{\gamma B^{2j}}{2} \pm B^{j+0.6}.
\end{equation*}

Now consider the second and third items.
If $\ell > j$, then by Chernoff-Hoeffding bound,
\begin{equation}
\label{eq:b-2}
\Pr\left[S_\ell \le \left(\frac{1}{2} - B^{-(j+1)}\right)  \gamma B^{2j} - \sqrt{10 \gamma \ln n} B^j \right] \le \Pr\left[S_\ell \le \bE[S_\ell] - \sqrt{10 \gamma B^{2j} \ln n}\right]  \le 1/n^{10}.
\end{equation}
If $\ell \le j$, then
\begin{equation}
\label{eq:b-3}
\Pr\left[S_\ell \ge \left(\frac{1}{2} - B^{-j}\right)  \gamma B^{2j} + \sqrt{10 \gamma \ln n} B^j \right] \le \Pr\left[S_\ell \ge \bE[S_\ell] + \sqrt{10 \gamma B^{2j} \ln n}\right]  \le 1/n^{10}.
\end{equation}
Since $B \ge (\ln n)^{100}$, we have
\begin{equation}
\label{eq:b-5}
\left(\frac{1}{2} - B^{-j}\right)  \gamma B^{2j} + \sqrt{10 \gamma \ln n} B^j < \zeta_1 =  \left(\frac{1}{2} - B^{-(j+1)}\right) \gamma B^{2j} - \sqrt{10 \gamma \ln n} B^j.
\end{equation}
The last two items follows from (\ref{eq:b-2}), (\ref{eq:b-3}) and (\ref{eq:b-5}).

\end{proof}

\subsection{Proof of Lemma~\ref{lem:use-ineq}}

\begin{proof}
Note that when $x \geq \min_{i \in [K]} \left\{ \frac{1 - \gamma_i}{ \gamma_i } \right\}$, the Left-Hand Side (LHS) of the desired inequality becomes $0$ and the RHS is less or equal to $0$. Therefore, we only need to prove the inequality assuming $x < \min_{i \in [K]} \left\{ \frac{1 - \gamma_i}{ \gamma_i } \right\}$.

Now the LHS becomes $\prod_{i = 1}^K (1 - \gamma_i - \gamma_i x)$. Let $f(t) = \prod_{i = 1}^K (1 - \gamma_i - \gamma_i t)$ for $t \in [0, x]$. Note that $f'(t) = - \sum_{i = 1}^K \gamma_i \prod_{j \neq i} (1 - \gamma_j - \gamma_j t) \geq f'(0)$ for $t \in [0, x]$. We have
\begin{multline*}
\prod_{i = 1}^K (1 - \gamma_i - \gamma_i x) = f(x) \geq f(0) + f'(0) x 
= \prod_{i = 1}^K (1 - \gamma_i) - \left( \sum_{i = 1}^K \gamma_i \prod_{j \neq i} (1 - \gamma_j ) \right) x  \\\geq \prod_{i = 1}^K (1 - \gamma_i) - \left(\prod_{i = 1}^{K}  \left( \gamma_i + (1 - \gamma_i) \right) \right) x = \prod_{i = 1}^K (1 - \gamma_i) - x.
\end{multline*}
\end{proof}

\end{document}